\documentclass[10pt]{article}

\usepackage{graphicx}
\usepackage{amstext}

\usepackage{amsmath, amsthm, amssymb}
\usepackage{appendix}
\usepackage[UKenglish]{babel}
\usepackage{proof}
\usepackage{bussproofs}
\usepackage{fancyhdr}
\usepackage{hyperref}
\usepackage[cmtip,arrow]{xy}
\usepackage{pb-diagram,pb-xy}
\usepackage{mathrsfs}
\usepackage{algorithm}
\usepackage[noend]{algpseudocode}
\usepackage{mathrsfs}
\usepackage{color}
\usepackage{cite}
\usepackage{courier} % code
\usepackage{latexsym}
\usepackage{amsfonts}
\usepackage{wasysym}
\usepackage{microtype}
\usepackage{verbatim}

\numberwithin{equation}{section}
\numberwithin{figure}{section}

\newtheorem{thm}{Theorem}[section]
\newtheorem{cor}[thm]{Corollary}
\newtheorem{lem}[thm]{Lemma}
\theoremstyle{remark}

\theoremstyle{definition}
\newtheorem{rem2}{Definition}[section]
\newtheorem{eg}{Example}

\newcommand{\snakes}{\mid\!\sim}

\newcommand{\pair}[1]{\left({#1}\right)}
\newcommand{\ang}[1]{\left\langle{#1}\right\rangle}
\newcommand{\set}[1]{\left\{{#1}\right\}}

\newcommand{\es}{\varnothing}

\newcommand{\pow}{\mathcal{P}}
\newcommand{\nat}{\mathbb{N}}

\newcommand{\x}{\times}
\newcommand{\sqbra}[1]{\left[{#1}\right]}

\newcommand{\ncsnakes}{\not\kern-0.01cm\mid\!\sim^{\bf{C}}}

\newcommand{\npsnakes}{\not\kern-0.01cm\mid\!\sim^{\bf{P}}}

\newcommand{\nrsnakes}{\not\kern-0.01cm\mid\!\sim^{\bf{R}}}

\newcommand{\lang}{\mathcal{L}}

\newcommand{\LSent}{\mathcal{SL}}

\newcommand{\alg}{\mathcal{A}}

\newcommand{\LForm}{\mathcal{FL}}

\newcommand{\relsymb}{\mathcal{R}}

\newcommand{\orde}{\trianglelefteq_{Eli}}
\newcommand{\ordd}{\trianglelefteq_{Dem}}
\newcommand{\ext}{\mathcal{E}}

\newcommand{\ordeneq}{\triangleleft_{Eli}}
\newcommand{\orddneq}{\triangleleft_{Dem}}
\newcommand{\powfin}{\pow_\text{fin}}
\newcommand{\subarg}{\subseteq_\text{arg}}
\newcommand{\ord}{\trianglelefteq}
\newcommand{\ordneq}{\triangleleft}
\newcommand{\attk}{\rightharpoonup}
\newcommand{\defeat}{\hookrightarrow}
\newcommand{\propsubarg}{\subset_\text{arg}}
\newcommand{\conflict}{\mathcal{C}}
\newcommand{\ordde}{\trianglelefteq_{DEli}}
\newcommand{\orddeneq}{\triangleleft_{DEli}}

\begin{document}

\title{Argumentation Semantics for Prioritised Default Logic}
\author{Anthony P. Young\footnote{\textbf{Corresponding author:} Peter Young, Department of Informatics, King's College London, \href{mailto:anthony.p.young@kcl.ac.uk}{anthony.p.young@kcl.ac.uk}}, Sanjay Modgil, Odinaldo Rodrigues}
\date{1st July 2015}
\maketitle

\begin{abstract}
\noindent We endow prioritised default logic (PDL) with argumentation semantics using the ASPIC$^+$ framework for structured argumentation, and prove that the conclusions of the justified arguments are exactly the prioritised default extensions. Argumentation semantics for PDL will allow for the application of argument game proof theories to the process of inference in PDL, making the reasons for accepting a conclusion transparent and the inference process more intuitive. This also opens up the possibility for argumentation-based distributed reasoning and communication amongst agents with PDL representations of mental attitudes.
\end{abstract}

\tableofcontents

\section{Introduction}\label{sec:intro}

\noindent Dung's argumentation theory \cite{Dung:95} has become established as a general framework for non-monotonic reasoning (NMR). Given a set of well-formed formulae (wffs) $\Delta$ in some non-monotonic logic (NML), the arguments and attacks defined by $\Delta$ instantiate a Dung argumentation framework. Additionally, a preference relation over the defined arguments can be used to determine which attacks succeed as defeats. The justified arguments are then evaluated under various Dung semantics, and the claims of the sceptically justified arguments\footnote{i.e. the arguments contained in \textit{all} extensions under some semantics.} identify the inferences from the underlying $\Delta$.

More formally, given an argumentation framework $AF$ and a wff $\theta$, the \emph{argumentation-defined inference relation} $\snakes_{AF}$ over $\Delta$ is $\Delta\snakes_{AF}\theta$ iff $\theta$ is the conclusion of a sceptically justified argument in $AF$. Indeed, a correspondence has been shown between $\snakes_{AF}$ over $\Delta$, and the instantiating logic's non-monotonic inference relation defined directly over $\Delta$. For example, Reiter's default logic (DL) \cite{Dung:95}, logic programming \cite{Dung:95}, defeasible logic \cite{Governatori:04} and Brewka's preferred subtheories \cite{sanjay:13} have all been been endowed with \emph{argumentation semantics}. This in turn allows the application of argument game proof theories \cite{Sanjay:09} to the process of inference, and the generalisation of these dialectical proof theories to distributed NMR amongst computational agents, whereby agents can engage in argumentation-based dialogues, submitting arguments and counter-arguments from their own non-monotonic knowledge bases \cite{Added_Value,Sanjay:08,Atkinson:05}. Furthermore, argumentative characterisations of NMR make use of principles familiar in everyday reasoning and debate, thus rendering transparent the reasons for accepting a conclusion and allowing for human participation and inspection of the inference process.

One well-studied NML that has not yet been endowed with argumentation semantics is Brewka's \emph{prioritised default logic} (PDL) \cite{Brewka:94}. PDL is important because it upgrades DL \cite{Reiter:80} with priorities over defaults\footnote{So that, for example, one can account for recent information taking priority over information in the distant past, or that more specific information should take priority over more general information.}. PDL has also been used to represent the (possibly conflicting) beliefs, obligations, intentions and desires (BOID) of agents, and model how these different categories of mental attitudes override each other in order to generate goals and actions that attain those goals \cite{BOID:02}.

In this note, we endow PDL with argumentation semantics and prove a correspondence between PDL inference and the inference relation defined by the argumentation semantics. We achieve this by appropriately instantiating the ASPIC$^+$ framework for structured argumentation \cite{sanjay:13,Sanjay:14}. ASPIC$^+$ identifies conditions under which logics and preference relations instantiating Dung's frameworks satisfy the rationality postulates of \cite{Caminada:07}. In Section \ref{sec:rev}, we review ASPIC$^+$ and PDL. In Section \ref{sec:ASPIC+_eli_ord}, we identify an error in the proof of \cite[page 376, Proposition 21]{sanjay:13}, which forces us to slightly modify our choice of argument preference relations in Section \ref{sec:rev_ASPIC+}. In Section \ref{sec:ASPIC+_to_PDL}, we define a PDL instantiation of ASPIC$^+$. This will involve studying preference relations thus far not considered by ASPIC$^+$. We then present a representation theorem proving that inferences defined by the argumentation semantics correspond exactly to inferences in PDL. In Section \ref{sec:normative_rationality_current}, we discuss to what extent is this ASPIC$^+$ instantiation normatively rational. Finally, in Section \ref{sec:discussion_conclusions}, we conclude with future work and some remarks about our approach.

In Appendix \ref{app:disj_eli_props}, we prove some properties of a non-ASPIC$^+$ argument preference relation, that is relevant to the PDL instantiation. In Appendix \ref{app:NBD_char}, we prove an intermediate result necessary for the proof of the representation theorem.

\section{Background}\label{sec:rev}

\noindent In the remainder of this paper we make use of the following notation: $\nat$ is the set of natural numbers, with $\nat^+:=\nat-\set{0}$. For a set $X$ its power set is $\pow(X)$, the set of its finite subsets is $\powfin(X)$, so $X\subseteq_\text{fin}Y$ iff $X$ is a finite subset of $Y$, therefore $X\in\powfin(Y)\Leftrightarrow X\subseteq_\text{fin}Y$. Undefined quantities are denoted by $*$, for example $1/0=*$ in the real numbers. If $\ang{P,\:\leq}$ is a preset (preordered set) then the strict version of the preorder is $a<b\Leftrightarrow\sqbra{a\leq b,\:b\not\leq a}$, which is easily shown to be a strict partial order. For two sets $A,\:B$ their symmetric difference is $A\ominus B:=\pair{A-B}\cup\pair{B-A}$.

\subsection{Dung's Abstract Argumentation Theory}\label{sec:rev_Dung}

\noindent We now recap the important definitions of \cite{Dung:95}. A \emph{(Dung) argumentation framework} is a directed graph $\ang{\alg,\:\conflict}$, where $\alg$ is the \emph{set of arguments} and $\conflict\:\subseteq\:\alg^2$ is the \emph{conflict relation} over $\alg$. For arguments $A,\:B\in\alg$ we write $\conflict(A,\:B)\Leftrightarrow(A,\:B)\in\conflict\Leftrightarrow A$ conflicts with $B$, i.e. $A$ is used as a counterargument against $B$. Note that $\conflict$ can denote either an attack relation defined by a set of instantiating formulae, or the defeat relation defined by determining which attacks succeed as defeats.

In what follows let $S\subseteq\alg$ be a set of arguments and $A,\:B\in\alg$. $S$ \emph{conflicts with} $B$ iff $\pair{\exists A\in S}\conflict(A,\:B)$. $S$ is \emph{conflict-free} (cf) iff $\conflict\cap S^2=\es$. $S$ \emph{defends} $A$ iff $\pair{\forall B\in\alg}[\conflict(B,\:A)\Rightarrow S$  conflicts with $B]$. Let $Def(S):=\set{A\in\alg\:\vline\:S\text{ defends }A}$. Then, $S$ is an \emph{admissible set} iff $S$ is cf and $S\subseteq Def(S)$. An admissible set $S$ is:

\begin{itemize}
\item a \emph{complete extension} iff $Def(S)\subseteq S$;
\item a \emph{preferred extension} iff $S$ is a $\subseteq$-maximal complete extension;
\item the \emph{grounded extension} iff $S$ is the $\subseteq$-least complete extension;
\item a \emph{stable extension} iff $S$ is complete and conflicts with all arguments in $\alg - S$.
\end{itemize} 

\noindent $\mathcal{S}:=\{complete,preferred,grounded,stable\}$ is the set of \emph{Dung semantics}. An argument $A\in\alg$ is \emph{sceptically justified} under the semantics $s\in\mathcal{S}$ iff $A$ belongs to all $s$ extensions.

\subsection{The ASPIC\texorpdfstring{$^+$}{+} Framework}\label{sec:rev_ASPIC+}

\noindent Dung's framework provides an intuitive calculus of opposition for determining the justified arguments based on conflict alone. However, it abstracts from the internal logical structure of arguments, the nature of defeats and how they are determined by preferences, and consideration of the conclusions of the arguments. However, these features are referenced when studying whether any given logical instantiation of a framework yields complete extensions that satisfy the rationality postulates of \cite{Caminada:07}. ASPIC$^+$ \cite{sanjay:13,Sanjay:14} provides a structured account of abstract argumentation, allowing one to reference the above features, while at the same time accommodating a wide range of instantiating logics and preference relations. ASPIC$^+$ then identifies conditions under which the instantiation (given arguments, attacks and preferences) results in complete extensions that satisfy the rationality postulates of \cite{Caminada:07}; such instantiations are \emph{normatively rational}.

\subsubsection{Construction of Arguments and Knowledge Bases}\label{sec:arg_constr_ASPIC+}

An \emph{(ASPIC$^+$) argumentation system} is $\ang{\lang,\:-,\:\relsymb_s,\:\relsymb_d,\:n}$ where $\lang$ is a logical language, $-:\lang\to\pow(\lang)$ is the \emph{contrary function} $\theta\mapsto\overline{\theta}$ that identifies when one wff in $\lang$ conflicts with another. Let $\theta_1,\dots,\theta_m, \phi\in\lang$ be wffs for $m\in\nat$, $\relsymb_s$ is the \emph{set of strict inference rules} of the form $(\theta_1,\:\ldots,\:\theta_m\to\phi)$, denoting that if $\theta_1,\:\ldots,\:\theta_m$ are true then $\phi$ is true no matter what, and $\relsymb_d$ is the \emph{set of defeasible inference rules} of the form $(\theta_1,\:\ldots,\:\theta_m\Rightarrow\phi)$, denoting that if $\theta_1,\:\ldots,\:\theta_m$ are true then $\phi$ is true, unless there are good reasons not to accept $\phi$. Finally $n:\relsymb_d\to\lang$ is a \emph{partial} function that assigns a \emph{name} to \emph{some} of the defeasible rules.

For each rule $r\in\relsymb_s\cup\relsymb_d$ we define two functions:
\begin{enumerate}
\item The \textit{antecedent map} is
\begin{align}\label{eq:antecedent_map_for_rules}
Ante:\relsymb&\to\powfin(\lang)\nonumber\\
r:=(\theta_1,\:\cdots,\:\theta_m\to/\Rightarrow\phi)&\mapsto Ante(r):=\set{\theta_1,\:\cdots,\:\theta_m}\:.
\end{align}
Note that $Ante$ returns a \textit{finite set} of formulae.
\item The \textit{consequent map} is
\begin{align}\label{eq:consequent_map_for_rules}
Cons:\relsymb&\to\lang\nonumber\\
r:=(\theta_1,\:\cdots,\:\theta_m\to/\Rightarrow\phi)&\mapsto Cons(r):=\phi\:.
\end{align}
\end{enumerate}

\noindent The names of the rules are unique, i.e. if $Ante(r)=Ante(r')$ and $Cons(r)=Cons(r')$, then $r=r'$. By equality we mean \textit{syntactic} equality with respect to the instantiating logic, e.g. if $\lang$ is propositional logic, $\neg(\theta\to\phi)\neq(\theta\wedge\neg\phi)$. Two rules $r,\:r'$ are \textit{equal} iff $Ante(r)=Ante(r')$ (syntactically) and $Cons(r)=Cons(r')$ (syntactically).

An \emph{(ASPIC$^+$) knowledge base} is a set $\mathcal{K}:=\mathcal{K}_n\cup\mathcal{K}_p\subseteq\lang$ where $\mathcal{K}_n$ is the set of \emph{axioms} and $\mathcal{K}_p$ is the set of \emph{ordinary premises}. Intuitively, the knowledge base consists of the premises used in constructing arguments. Note that $\mathcal{K}$ does not have to be a finite set. Given an argumentation system and knowledge base, an \emph{(ASPIC$^+$) argument} is defined inductively:
\begin{enumerate}
\item(Base) $[\theta]$ is a \emph{singleton (ASPIC$^+$) argument} with $\theta\in\mathcal{K}$, conclusion $Conc([\theta]):=\theta$, premise set $\set{\theta}\subseteq\mathcal{K}$ and top rule $TopRule([\theta]):=*$.

\item(Inductive, strict) Let $1\leq i\leq n$ be an index. For each such $i$ let $A_i$ be an ASPIC$^+$ argument with conclusion $Conc(A_i)$ and premise set $Prem(A_i)$. If $r:=(Conc(A_1),\:\ldots,\:Conc(A_n)\to\phi)\in\relsymb_s$, then $B:=[A_1,\:\ldots,\:A_n\to\phi]$ is also an ASPIC$^+$ argument with conclusion $Conc(B)=\phi$, premise set $Prem(B):=\bigcup_{i=1}^n Prem(A_i)\subseteq\mathcal{K}$ and $TopRule(B)=r\in\relsymb_s$.

\item(Inductive, defeasible) Let $1\leq i\leq n$ be an index. For each such $i$ let $A_i$ be an ASPIC$^+$ argument with conclusion $Conc(A_i)$ and premise set $Prem(A_i)$. If $r':=(Conc(A_1),\:\ldots,\:Conc(A_n)\Rightarrow\phi)\in\relsymb_d$, then $C:=[A_1,\:\ldots,\:A_n\Rightarrow\phi]$ is also an ASPIC$^+$ argument with conclusion $Conc(C)=\phi$, premise set $Prem(C):=\bigcup_{i=1}^n Prem(A_i)\subseteq\mathcal{K}$ and $TopRule(C)=r'\in\relsymb_d$.
\end{enumerate}

\noindent Let $\alg$ be the (unique) set of all arguments constructed in this way. It is clear that arguments are finite objects.

\subsubsection{Properties of Arguments}

A \emph{subargument} $B$ of $A$ is (informally) an argument where $Prem(B)\subseteq Prem(A)$ and $Conc(B)$ appears as an intermediate conclusion of $A$ attained by the application of the exact same rules\footnote{One can formally define subarguments via how arguments are constructed as described in the previous section.}. Given an argument $A$, its \emph{set of subarguments} is $Sub(A)\subseteq\alg$ and its \emph{set of proper subarguments} is $Sub(A)-\set{A}\subseteq\alg$. We will also write $A\subarg B\Leftrightarrow A\in Sub(B)$, and $A\propsubarg B\Leftrightarrow A\in Sub(B)-\set{B}$. It is easy to show that $\subarg$ is a preorder and $\propsubarg$ is a strict partial order on $\alg$. Informally, two arguments are \emph{equal} iff they are constructed identically in the above manner\footnote{More formally, argument equality can be defined inductively given how arguments are constructed. This will make $\subarg$ into a partial order.}. Further, a set $S\subseteq\alg$ is \textit{subargment closed} iff $\pair{\forall A\in S}Sub(A)\subseteq S$.

An argument $A\in\alg$ is \emph{firm} iff $Prem(A)\subseteq\mathcal{K}_n$, i.e. all of its premises are axioms. Further, $DR(A)\subseteq\relsymb_d$ is the set of defeasible rules applied in constructing $A$. An argument $A$ is \emph{strict} iff $DR(A)=\es$, else $A$ is \emph{defeasible}.

We define the \textit{conclusion map}
\begin{align}
Conc:\alg&\to\lang\nonumber\\
A&\mapsto Conc(A)\:,
\end{align}
which takes an argument and returns its conclusion\footnote{Do not confuse $Conc:\alg\to\lang$ with $Cons:\relsymb\to\lang$ (Equation \ref{eq:consequent_map_for_rules}, page \pageref{eq:consequent_map_for_rules}).}. We can generalise this to \textit{sets} of arguments as follows:
\begin{align}\label{eq:set_arg_attrb}
Conc':\pow\pair{\alg}&\to\pow\pair{\lang}\nonumber\\
S&\mapsto Conc'(S):=\bigcup_{A\in S}Conc(A)\:.
\end{align}
We will abuse notation and not distinguish between $Conc'$ and $Conc$ if there is no danger of ambiguity. Similarly, one can also define $Prem:\alg\to\powfin\pair{\lang}$, $DR:\alg\to\powfin\pair{\relsymb_d}$, $TopRule:\alg\to\relsymb$ from how arguments are constructed, and generalise their domains from single arguments $\alg$ to sets of arguments $\pow\pair{\alg}$. Further, we define, for all $A\in\alg$, $Prem_n(A):=Prem(A)\cap\mathcal{K}_n$, $Prem_p(A):=Prem(A)\cap\mathcal{K}_p$, $SR(A)\subseteq\relsymb_s$ is the set of strict rules applied in constructing $A$.

Notice in the cases of $Prem$ and $DR$, the codomains are appropriate powersets and not finite powersets, this is because for \textit{arbitrary} sets of arguments, even though each argument is mapped to its finite set of premises or defeasible rules, the set of arguments can be infinite and the union of infinitely many finite sets does not have to be finite.

\subsubsection{Attacks}

An argument $A$ \emph{attacks} another argument $B$, denoted as $A\attk B$, iff at least one of the following hold, where:\footnote{See \cite[Section 2]{Sanjay:14} for a further discussion of why attacks are distinguished in this way.}
\begin{enumerate}
%\item  Let $Prem_p(B):=Prem(B)\cap\mathcal{K}_p$ be the set of ordinary premises of $B$.
\item $A$ is said to \emph{undermine} attack $B$ on the subargument $B'$ = $[\phi]$ iff\[\sqbra{\exists\phi\in Prem_p(B)}\:Conc(A)\in\overline{\phi}\:,\]i.e. $A$ conflicts with some ordinary premise of $B$.
\item There is some $B'\subarg B$ such that for all $i=1,\:\ldots,\:n$,\[\sqbra{B''_i\subarg B,\:B'=[B_1'',\:\ldots,\:B_n''\Rightarrow\phi]}\]and $Conc(A)\in\overline{\phi}$.  $A$ is then said to \emph{rebut} attack $B$ on the subargument $B'$.
\item There is some $B'\subarg B$ such that $r:=TopRule(B')\in\relsymb_d$ and $Conc(A)\in\overline{n(r)}$. $A$ is then said to \emph{undercut} attack $B$ on the subargument $B'$ (by arguing against the application of the defeasible rule $r$ in $B$).
\end{enumerate}

\noindent We then abuse notation to define the \emph{attack relation} as $\attk\:\subseteq\alg^2$ such that $(A,\:B)\:\in\:\attk\:\Leftrightarrow\: A\attk B$. A set of arguments $S\subseteq\alg$ is \emph{attack-conflict-free} (attack-cf) iff $S^2\:\cap\:\attk\:=\:\es$. By transitivity of $\subarg$, for all $A,\:B,\:C\in\alg$, if $A\attk B$ and $B\subarg C$, then $A\attk C$.

\subsubsection{Preferences and Defeats}

A preference relation over arguments is then used to determine which attacks  succeed as defeats. We denote the preference $\precsim\:\subseteq\:\alg^2$ (not necessarily a preorder for now) such that $A\precsim B\Leftrightarrow\:A$ is not more preferred than $B$. The strict version is $A\prec B\Leftrightarrow\sqbra{A\precsim B,\:B\not\precsim A}$, and equivalence is $A\approx B\Leftrightarrow\sqbra{A\precsim B,\:B\precsim A}$. We define a \emph{defeat} as
\begin{align}\label{eq:ASPIC+_general_defeat}
A\defeat B\Leftrightarrow\pair{\exists B'\subarg B}\sqbra{A\attk B',\:A\not\prec B'}\:.
\end{align}
That is to say, $A$ defeats $B$ (on $B'$) iff  $A$ attacks $B$ on the subargument $B'$, and $B'$ is not 
strictly preferred to $A$. Notice the comparison is made at the subargument $B'$ instead of the whole argument $B$. We then abuse notation to define the \emph{defeat relation} as $\defeat\:\subseteq\alg^2$ such that $(A,\:B)\:\in\:\defeat\:\Leftrightarrow\: A\defeat B$. By transitivity of $\subarg$, for all $A,\:B,\:C\in\alg$, if $A\defeat B$ and $B\subarg C$, then $A\defeat C$. A set of arguments $S\subseteq\alg$ is \textit{defeat-conflict-free} (defeat-cf) iff $S^2\cap\defeat\:=\:\es$.

What is the difference between attack-cf and defeat-cf? Clearly, attack-cf implies defeat-cf but the converse is not true. \cite{sanjay:13} argues that attack-cf should be the correct notion of conflict-freeness to use when invoking Dung semantics, because the presence of attacks denote disagreement between two arguments, and ideally an agent should not accept two arguments that disagree with each other. However, one can also argue that defeat-cf is the correct notion of conflict-freeness to use (e.g. \cite{Prakken:10}), because an agent can accept two arguments that attack each other, knowing also that the attack does not succeed due to the preference relation. In the upcoming sections, we will prove the stronger result of attack-cf whenever it is possible to, and defeat-cf follows.

In ASPIC$^+$, preferences over arguments are calculated from the argument structure through comparing the fallible information (ordinary premises and defeasible rules) they contain. More formally, $\mathcal{K}_p$ and $\relsymb_d$ are endowed with preorders $\leq'$ and $\leq''$ respectively\footnote{Where the bigger item is \emph{more preferred}.}. This preorder is then lifted to a set-comparison order $\trianglelefteq$ between the (finite) sets of premises or defeasible rules of the arguments, and then finally to $\precsim$, following the method in \cite[Section 5]{sanjay:13}. 

We now recap this lifting of the preorder $<''$ from $\relsymb_d$ to $\powfin\pair{\relsymb_d}$. We omit comparing premises because in our instantiation we only compare defeasible rules as there are no ordinary premises (Section \ref{sec:ASPIC+_to_PDL}, page \pageref{sec:ASPIC+_to_PDL}).

More formally, ASPIC$^+$ considers two ordering principles called \emph{democratic} and \emph{elitist} \cite[Section 5]{sanjay:13}\footnote{See \cite[Section 3.5]{Sanjay:14} for a further discussion of both these ordering principles.}, such that for $A,\:B\in\alg$ and $DR(A)\subseteq\relsymb_d$, we define $A\precsim B$ to be
\begin{align}
&DR(A)\orde DR(B)\label{eq:eli_comp}\\
\text{or }&DR(A)\ordd DR(B)\:,\label{eq:dem_comp}
\end{align}
where, for\footnote{It suffices to consider finite sets as arguments are finite.} $\Gamma,\:\Gamma'\subseteq_\text{fin}\relsymb_d$,
\begin{align}
\Gamma\orde\Gamma'&\Leftrightarrow\sqbra{\Gamma=\Gamma'\text{ or }\Gamma\ordeneq\Gamma'}\:,\label{eq:original_non-strict_elitist}\\
\Gamma\ordeneq\Gamma'&\Leftrightarrow\pair{\exists x\in\Gamma}\pair{\forall y\in\Gamma'}\:x<''y\:,\label{eq:original_elitist}\\
\Gamma\ordd\Gamma'&\Leftrightarrow\sqbra{\Gamma=\Gamma'\text{ or }\Gamma\orddneq\Gamma'}\:,\text{ and }\\
\Gamma\orddneq\Gamma'&\Leftrightarrow\pair{\forall x\in\Gamma}\pair{\exists y\in\Gamma'}\:x<''y\:,\label{eq:original_dem}
\end{align}
It is easy to show that $\precsim$ in both cases is a preorder on $\alg$. We define equivalence of arguments as follows:
\begin{align}\label{eq:arg_pref_equivalence}
A\approx B\Leftrightarrow DR(A)=DR(B)\:.
\end{align}
Note that Equations \ref{eq:original_non-strict_elitist} and \ref{eq:original_elitist} are not exactly the same as \cite[page 375, Definition 19]{sanjay:13}. We will explain this in Section \ref{sec:ASPIC+_eli_ord} (page \pageref{sec:ASPIC+_eli_ord}).

In summary, when comparing two arguments $A,\:B\in\alg$, $A\precsim B$ iff [$A\approx B$ (Equation \ref{eq:arg_pref_equivalence}) or $A\prec B$]. In the latter case, $A\prec B\Leftrightarrow DR(A)\ordeneq DR(B)$, or $A\prec B\Leftrightarrow DR(A)\orddneq DR(B)$, depending on which ordering principle is being used.

Given the preference relation $\precsim$ between arguments, we call the structure $\ang{\alg,\:\attk,\:\precsim}$ an \emph{ASPIC$^+$ SAF} (structured argumentation framework), or \emph{attack graph}. Its corresponding \emph{defeat graph} is $\ang{\alg,\:\defeat}$, where $\defeat$ is defined in terms of $\attk$ and $\precsim$ as in Equation \ref{eq:ASPIC+_general_defeat}.

\subsubsection{Applying Dung's Semantics}

Given $\ang{\alg,\:\defeat}$ one can then evaluate the extensions under Dung's semantics (Section \ref{sec:rev_Dung}, page \pageref{sec:rev_Dung}), and thus identify the argumentation defined inferences as the conclusions of the sceptically justified arguments as follows. Let $AS$ be an argumentation system. The \emph{argumentation-defined inference relation} $\snakes_{AS}$ is $\mathcal{K}\snakes_{AS}\theta$ iff $\theta=Conc(A)$ where $A\in\alg$ is a sceptically justified argument.

\subsubsection{Conditions for Normative Rationality}\label{sec:rev_ratl}

Instantiations of ASPIC$^+$ should satisfy some properties to ensure it is normatively rational \cite{Caminada:07}. Let $\ext\subseteq\alg$ be a complete extension. Informally, \emph{subargument closure} states that if an argument is in $\ext$, then all its subarguments are in $\ext$. \emph{Closure under strict rules} states that if the conclusions of arguments in $\ext$ strictly entail some $\phi$, then $\ext$ contains an argument concluding $\phi$. Finally, \emph{consistency} states that $Conc(\ext):=\bigcup_{A\in\ext} Conc(A)$ is a consistent with respect to the instantiating logic. Collectively these are the \emph{(Caminada-Amgoud) rationality postulates}.

ASPIC$^+$ then identifies sufficient conditions for an instantiation to satisfy these rationality postulates. These are that the instantiation is \emph{well-defined} and that the argument preference ordering $\precsim$ is \emph{reasonable}. We will say more about these conditions in Section \ref{sec:normative_rationality_current} (page \pageref{sec:normative_rationality_current}), where we discuss whether the ASPIC$^+$ characterisation of PDL satisfies the rationality postulates.

\subsection{Prioritised Default Logic}\label{sec:rev_PDL}

\subsubsection{First Order Logic}

In this section we recap PDL \cite{Brewka:94}. We work in full first order logic (FOL) where the set of first-order formulae is $\LForm$ and the set of closed first order formulae\footnote{i.e. first order formulae without free variables} is $\LSent\subseteq\LForm$, with the usual quantifiers and connectives. Given $S\subseteq\LForm$, the \emph{deductive closure (of $S$)} is $Th(S)$, and given $\theta\in\LForm$, the \emph{addition operator} $+:\pow(\LForm)\x\LForm$ is defined as $S+\theta:=Th(S\cup\set{\theta})$.

\subsubsection{Normal Defaults}

A \emph{normal default} is an expression $\frac{\theta:\phi}{\phi}$ where $\theta,\:\phi\in\LForm$ and read ``if $\theta$ is the case and $\phi$ is consistent with what we know, then $\phi$ is the case''\footnote{There are other possible interpretations of normal defaults, see Example \ref{eg:teaching}.}. In this case we call $\theta$ the \emph{antecedent} and $\phi$ the \emph{consequent}. A normal default $\frac{\theta:\phi}{\phi}$ is \emph{closed} iff $\theta,\:\phi\in\LSent$. We will assume all defaults are closed and normal unless stated otherwise. Given $S\subseteq\LSent$, a default is \emph{active (in $S$)} iff $\sqbra{\theta\in S,\:\phi\notin S,\:\neg\phi\notin S}$. Intuitively, the first requirement says we need to know the antecedent before applying the default, the second requirement is that the consequent must add new information, and the third requirement ensures that what we infer is consistent with what we know.

\subsubsection{Prioritised Default Theories and Extensions}

A \emph{finite prioritised default theory} (PDT) is a structure $\ang{D,\:W,\:\prec}$, where $W\subseteq\LSent$ is not necessarily a finite set and $\ang{D,\:\prec}$ is a \emph{finite} strict poset (partially ordered set) of defaults, where $d'\prec d\Leftrightarrow d$ is \emph{more\footnote{\label{fn:dual_priority_PDL} We have defined the order dually to \cite{Brewka:94} so as to comply with orderings over the ASPIC$^+$ defeasible inference rules.} prioritised} than $d'$. Intuitively, $W$ are the known facts and $D$ the defaults that nonmonotonically extend $W$. We will consider finite PDTs unless otherwise specified.

A PDT's inferences are defined by its extensions. Formally, let $\prec^+\supseteq\prec$ be a linearisation\footnote{i.e. $\prec^+$ is a strict total order and hence $\ang{D,\:\prec^+}$ is a strict toset (totally ordered set).} of $\prec$. An \emph{extension (with respect to $\prec^+$)} is a set $E:=\bigcup_{i\in\nat}E_i\subseteq\LSent$ built inductively as:
\begin{align}
E_0&:=Th(W)\text{ and }\label{eq:ext_base}\\
E_{i+1}&:=
\begin{cases}
E_i+\phi\:, &\text{condition 1}\\
E_i\:,&\text{else}
\end{cases}\label{eq:ext_ind}
\end{align}
where ``condition 1'' iff ``$\phi$ is the consequent of the $\prec^+$-greatest default $d$ active in $E_i$''. Intuitively, one first generates all classical consequences from the facts $W$, and then iteratively adds the nonmonotonic consequences from the most prioritised default to the least. Notice if $W$ is inconsistent then $E_0=E=\LForm$. We will assume $W$ is consistent unless stated otherwise.

It can be shown that the ascending chain $E_i\subseteq E_{i+1}$ stabilises at some finite $i\in\nat$ and that $E$ is consistent provided that $W$ is consistent. $E$ does not have to be unique because there may be more than one distinct linearisation $\prec^+$ of $\prec$. We say the PDT $\ang{D,\:W,\:\prec}$ \emph{sceptically infers} $\theta\in\LSent$ iff $\theta\in E$ for all extensions $E$.

Henceforth, we will refer to a PDT $\ang{D,\:W,\:\prec}$ where $\prec$ is a strict total order as a \emph{linearised PDT} (LPDT). If $\prec$ is already total then there is only one way to apply the defaults in $D$ (Equation \ref{eq:ext_ind}), hence the extension is unique and all inferences are sceptical. In what follows, we will use $\prec^+$ to emphasise that the order is total.

One application of PDL is in modelling how an agent reasons with her beliefs, obligations, intentions and desires (BOID).

\begin{eg}\label{eg:teaching}
Suppose a research assistant Alice ($a$) is considering whether she should teach undergraduates. We can model her mental attitudes as a BOID agent's PDT \cite{BOID:02} as follows. Define the predicates $R(x)\Leftrightarrow\:``x$ is a research assistant'', $A(x)\Leftrightarrow\:``x$ is an academic'', and $T(x)\Leftrightarrow\:``x$ is teaching (undergraduates)''. Alice is a research assistant, so $W=\set{R(a)}$. She believes that research assistants are academics, so her set of beliefs $Bel$ has the default $\frac{R(a):A(a)}{A(a)}$. She does not want to teach and would rather focus on her research, so her set of desires $Des$ include $\frac{R(a):\neg T(a)}{\neg T(a)}$. However, she is obliged to teach, so her set of obligations $Obl$ include $\frac{A(a):T(a)}{T(a)}$. The set of defaults is $D=Bel\cup Des\cup Obl$, and we assume no other defaults are relevant for this example.

In \cite{BOID:02}, the relative prioritisations of categories of mental attitudes define different agent types. For example, if Alice is a \emph{realistic selfish agent}, the priority (abuse of notation) is $Obl\prec^+ Des\prec^+ Bel$, and therefore the extension is $Th\pair{\set{R(a),\:A(a),\:\neg T(a)}}$. She thus generates the goal $\neg T(a)$, i.e. she does not teach. However, if she is a \emph{realistic social agent}, the priority (abuse of notation) is $Des\prec^+ Obl\prec^+ Bel$, and therefore she teaches, as $T(a)$ is in the extension.
\end{eg}

\section{The Corrected ASPIC\texorpdfstring{$^+$}{+} Preferences}\label{sec:ASPIC+_eli_ord}

\subsection{Overview of the Problem}

In this section, we show that the elitist set comparison relation \cite[page 375, Definition 19]{sanjay:13} is not reasonable inducing and hence cannot guarantee normative rationality. This is because there is an error in the proof of \cite[page 376, Proposition 21]{sanjay:13}. One should use Prakken's original strict elitist set comparison relation instead \cite{Prakken:10}, which does result in normatively rational ASPIC$^+$ instantiations. This is why Equation \ref{eq:original_elitist} (page \pageref{eq:original_elitist}) is not the same as the original ASPIC$^+$ elitist set comparison relation.

\subsection{The Elitist Set Comparison Relation is not Reasonable Inducing}

Recall that the property of \textit{reasonable inducing} for a given set comparison relation $\ord$ is necessary for ASPIC$^+$ instantiations to have normatively rational argument preference relations $\precsim$ based on $\ord$, because they preserve Dung's fundamental lemma \cite[page 327, Lemma 10]{Dung:95}, as discussed in \cite[Section 4.2]{sanjay:13}.

\begin{rem2}\label{def:reas_ind}
(From \cite[page 376, Definition 22]{sanjay:13}) Given a preset $\ang{P,\leq}$, a set comparison $\ord\:\subseteq\sqbra{\powfin(P)}^2$ is \textit{reasonable inducing} iff
\begin{enumerate}
\item $\ord$ is transitive.
\item For any $\Gamma_0,\:\Gamma_1,\:\cdots,\:\Gamma_n\subseteq_\text{fin}P$ (for $n\geq 1$), if
\begin{align}\label{eq:ass}
\bigcup_{i=1}^n\Gamma_i\ordneq\Gamma_0\text{ then}
\end{align}
\begin{enumerate}
\item $\pair{\exists 1\leq i\leq n}\Gamma_i\ord\Gamma_0$ and
\item $\pair{\exists 1\leq i\leq n}\Gamma_0\not\ord\Gamma_i$.
\end{enumerate}
\end{enumerate}
\end{rem2}

\noindent Recall the following identity for bounded existential quantifiers: for any unary predicate $\pow$ and a family of sets $A_i$ indexed by another set $I$,
\begin{align}\label{eq:exists_union}
\pair{\exists x\in\bigcup_{i\in I}A_i}\pow(x)&\Leftrightarrow\pair{\exists i\in I}\pair{\exists x\in A_i}\pow(x)\:.
\end{align}

\noindent We now explain why the proof that $\orde$ is reasonable inducing \cite[page 376, proposition 21]{sanjay:13} is incorrect by locating the error.

\begin{lem}
(The following statement, \cite[page 376, Proposition 21]{sanjay:13}, may not be true) $\orde$ is reasonable inducing.
\end{lem}
\begin{proof}
(The following proof, from \cite[page 390, Proposition 21]{sanjay:13}, is incorrect) We know that $\orde$ is transitive and satisfies Property 2(a) of reasonable inducing (Definition \ref{def:reas_ind}, page \pageref{def:reas_ind}) from Equation \ref{eq:exists_union}. Assume for contradiction that property 2(b) is false, i.e.
\begin{align}
&\pair{\forall 1\leq i\leq n}\Gamma_0\orde\Gamma_i\nonumber\\
\Leftrightarrow&\pair{\forall 1\leq i\leq n}\pair{\exists x\in\Gamma_0}\pair{\forall y\in\Gamma_i}x\leq y\nonumber\\
\Rightarrow&\pair{\exists x\in\Gamma_0}\pair{\forall y\in\Gamma_1}x\leq y\text{ (by choosing $i=1$),}\nonumber\\
\Leftrightarrow&\pair{\forall y\in\Gamma_1}x_1\leq y\:,\label{eq:temp}
\end{align}
where $x_1\in\Gamma_0$ is the witness to $\exists$. Now from the assumption of strictly less than in the set comparison relation (Equation \ref{eq:ass}, page \pageref{eq:ass}), we have
\begin{align}
\Gamma_0\not\orde\bigcup_{i=1}^n\Gamma_i&\Leftrightarrow\pair{\forall x\in\Gamma_0}\pair{\exists y\in\bigcup_{i=1}^n\Gamma_i}x\not\leq y\nonumber\\
&\Leftrightarrow\pair{\forall x\in\Gamma_0}\pair{\exists 1\leq i\leq n}\pair{\exists y\in\Gamma_i}x\not\leq y\text{ by Equation \ref{eq:exists_union}}\nonumber\\
&\Rightarrow\pair{\exists 1\leq i(x_1)\leq n}\pair{\exists y\in\Gamma_{i(x_1)}}x_1\not\leq y\:,\label{eq:temp2}
\end{align}
where in the last step we have instantiated $x$ under $\forall$ to $x_1\in\Gamma_0$, which is the witness to Equation \ref{eq:temp}.

(INCORRECT STEP) Assume the witness to $\exists$ in Equation \ref{eq:temp2} is 1, i.e. $i(x_1)=1$. Of course, there is no guarantee that the witness to the first $\exists$ in Equation \ref{eq:temp2} is the same as the instantiation of the first $\forall$ in ``$\pair{\forall 1\leq i\leq n}\Gamma_0\orde\Gamma_1$''.

Running with this, we have from Equations \ref{eq:temp} and \ref{eq:temp2},
\begin{align}
\pair{\forall y\in\Gamma_1}x_1\leq y\text{ and }\pair{\exists y\in\Gamma_1}x_1\not\leq y\:.
\end{align}
Therefore, by instantiating the first quantifier $\forall$ to the witness of the second quantifier $\exists$, calling it $y_0\in\Gamma_1$, we have
\begin{align}
x_1\leq y_0\text{ and }x_1\not\leq y_0\:,
\end{align}
which is the purported contradiction. Therefore, $\pair{\exists 1\leq i\leq n}\Gamma_0\not\orde\Gamma_i$.
\end{proof}

An incorrect proof does not mean that the elitist set comparison relation \cite[page 375, Definition 19]{sanjay:13} is not reasonable inducing. We now show that it is not reasonable inducing with the following counterexample.

\begin{lem}\label{lem:eli_reas_ind_second_prop_false}
The proposition ``if $\bigcup_{i=1}^n\Gamma_i\ordeneq\Gamma_0$ then $\pair{\exists 1\leq i\leq n}\Gamma_0\not\orde\Gamma_i$'' is false, i.e. Property 2(b) of Definition \ref{def:reas_ind} fails for $\orde$.
\end{lem}
\begin{proof}
The counterexample is as follows: let $\ang{P,\leq}$ be an arbitrary preset such that $a,b,c,d\in P$. Let $\Gamma_0=\set{c,d}$, $\Gamma_1=\set{a}$, $\Gamma_2=\set{b}$ so $\Gamma_1\cup\Gamma_2=\set{a,b}$. Let $\leq$ be such that $a\approx c,\:a<d,\:d\leq b$ and $c\not\leq b$. Notice
\begin{align*}
\Gamma_1\cup\Gamma_2\ordeneq\Gamma_0&\Leftrightarrow\sqbra{\Gamma_1\cup\Gamma_2\orde\Gamma_0\text{ and }\Gamma_0\not\orde\Gamma_1\cup\Gamma_2}\\
&\Leftrightarrow\sqbra{\set{a,b}\orde\set{c,d}\text{ and }\set{c,d}\not\orde\set{a,b}}\\
\set{a,b}\orde\set{c,d}&\Leftrightarrow\sqbra{\pair{a\leq c,\:a\leq d}\text{ or }\pair{b\leq c,\:b\leq d}}\:.\\
&\Leftrightarrow\text{true as }\sqbra{a\approx c\Rightarrow a\leq c\text{ and }a<d\Rightarrow a\leq d}\:.\\
\set{c,d}\not\orde\set{a,b}&\Leftrightarrow\sqbra{\pair{c\not\leq a\text{ or }c\not\leq b}\text{ and }\pair{d\not\leq a\text{ or } d\not\leq b}}\\
&\Leftrightarrow\text{true as }\sqbra{c\not\leq b\text{ and }a<d\Rightarrow d\not\leq a}\:.\\
\Gamma_0\orde\Gamma_1&\Leftrightarrow\set{c,d}\orde\set{a}\\
&\Leftrightarrow\sqbra{c\leq a\text{ or }d\leq a}\\
&\Leftrightarrow\text{true as }\sqbra{a\approx c\Rightarrow c\leq a}\:.\\
\Gamma_0\orde\Gamma_2&\Leftrightarrow\set{c,d}\orde\set{b}\\
&\Leftrightarrow\sqbra{c\leq b\text{ or }d\leq b}\\
&\Leftrightarrow\text{true because }d\leq b\:.
\end{align*}
Therefore, we have found a situation where $\Gamma_1\cup\Gamma_2\ordeneq\Gamma_0$, $\Gamma_0\orde\Gamma_1$ and $\Gamma_0\orde\Gamma_2$ are all true.
\end{proof}

\begin{cor}
$\orde$ is not reasonable inducing.
\end{cor}
\begin{proof}
Immediate from Definition \ref{def:reas_ind} (page \pageref{def:reas_ind}) and Lemma \ref{lem:eli_reas_ind_second_prop_false} (page \pageref{lem:eli_reas_ind_second_prop_false}).
\end{proof}

\noindent This failure of the property of being reasonable inducing allows for counterexamples like \cite[Example 5.1]{Dung:14}. In that example, two defeasible rules can be equivalent under a suitable preorder without being equal. The original elitist order from \cite{sanjay:13} does allow for defeasible to be equivalent (i.e. just as preferred as each other) without being equal.

\subsection{The Strict Elitist Set Comparison is Reasonable Inducing}

Consider the strict version of the elitist order, as originally proposed by Prakken in \cite[page 109]{Prakken:10}. We will show that it is reasonable inducing, at the cost of not allowing distinct defeasible rules to be equivalent, i.e. that the only notion of equivalence is equality. Recall that given a preset $\ang{P,\leq}$, its strict counterpart preorder is $a<b\Leftrightarrow\sqbra{a\leq b,\:b\not\leq a}$, which is a strict partial order.

\begin{rem2}\label{def:strict_eli_set_comparison}
Let $\ang{P,\leq}$ be a preset and form its strict poset $\ang{P,<}$. Define the \textit{strict elitist set comparison} $\ordeneq'$ on $\powfin(P)$ as
\begin{align}
\Gamma\ordeneq'\Gamma'\Leftrightarrow\pair{\exists x\in\Gamma}\pair{\forall y\in\Gamma'}x<y\:.
\end{align}
Its \textit{non-strict counterpart} is
\begin{align}
\Gamma\orde'\Gamma'\Leftrightarrow\sqbra{\Gamma=\Gamma'\text{ or }\Gamma\ordeneq'\Gamma'}\:.
\end{align}
\end{rem2}

\begin{cor}
$\ordeneq'$ is irreflexive.
\end{cor}
\begin{proof}
Assume for contradiction that $\Gamma\ordeneq'\Gamma$, which is equivalent to\[\pair{\exists x\in\Gamma}\pair{\forall y\in\Gamma}x<y\:.\]Let $x_0\in\Gamma$ be the witness to $\exists$, which means $\pair{\forall y\in\Gamma}x_0<y$, and one can instantiate $y$ to $y=x_0$, which means $x_0<x_0$ and hence a contradiction. Therefore, $\ordeneq'$ is irreflexive.
\end{proof}

\begin{lem}
The strict elitist set comparison (Definition \ref{def:strict_eli_set_comparison}) is reasonable inducing.
\end{lem}
\begin{proof}
Following Definition \ref{def:reas_ind}, we have:
\begin{enumerate}
\item Transitivity:
\begin{align*}
&\Gamma\ordeneq'\Gamma'\ordeneq'\Gamma''\\
\Leftrightarrow&\pair{\exists x\in\Gamma}\pair{\forall y\in\Gamma'}x<y\text{ and }\pair{\exists y\in\Gamma'}\pair{\forall z\in\Gamma''}y<z\\
\Leftrightarrow&\pair{\forall y\in\Gamma'}x_0<y\text{ and }\pair{\forall z\in\Gamma''}y_0<z\\
\Rightarrow&x_0<y_0\text{ and }\pair{\forall z\in\Gamma''}y_0<z\\
\Leftrightarrow&\pair{\forall z\in\Gamma''}x_0<y_0<z\\
\Rightarrow&\pair{\forall z\in\Gamma''}x_0<z\\
\Leftrightarrow&\pair{\exists x\in\Gamma}\pair{\forall z\in\Gamma''}x<z\\
\Leftrightarrow&\Gamma\ordeneq'\Gamma''\:,
\end{align*}
where in the third line $x_0\in\Gamma$ is the witness to the first $\exists$, and $y_0\in\Gamma'$ is the witness to the second $\exists$. Therefore, $\orde'$ is transitive.
\item Definition \ref{def:reas_ind}, Property 2(a): we have
\begin{align}
&\bigcup_{i=1}^n\Gamma_i\ordeneq'\Gamma_0\nonumber\\
\Leftrightarrow&\pair{\exists x\in\bigcup_{i=1}^n\Gamma_i}\pair{\forall y\in\Gamma_0}x<y\nonumber\\
\Leftrightarrow&\pair{\exists 1\leq i\leq n}\pair{\exists x\in\Gamma_i}\pair{\forall y\in\Gamma_0}x<y\text{ by Equation \ref{eq:exists_union}}\label{eq:strict_eli_middle}\\
\Leftrightarrow&\pair{\exists 1\leq i\leq n}\Gamma_i\ordeneq'\Gamma_0\text{ by Definition \ref{def:strict_eli_set_comparison}}\nonumber\\
\Rightarrow&\pair{\exists 1\leq i\leq n}\Gamma_i\orde'\Gamma_0\:.\nonumber
\end{align}
Therefore, $\orde'$ satisfies the first property.
\item Definition \ref{def:reas_ind}, Property 2(b): let $1\leq i_0\leq n$ be the witness to the first $\exists$ in Equation \ref{eq:strict_eli_middle}, and $x_{i_0}\in\Gamma_{i_0}$ be the witness to the second $\exists$ in Equation \ref{eq:strict_eli_middle}. Equation \ref{eq:strict_eli_middle}:
\begin{align}\label{eq:strict_eli_prop_2b1}
&\pair{\forall y\in\Gamma_0}x_{i_0}<y\:.
\end{align}
Now assume for contradiction that
\begin{align}
\pair{\forall 1\leq i\leq n}\Gamma_0\orde'\Gamma_i\Rightarrow&\Gamma_0\orde'\Gamma_{i_0}\nonumber\\
\Leftrightarrow&\pair{\exists x\in\Gamma_0}\pair{\forall y\in\Gamma_{i_0}}x<y\nonumber\\
\Leftrightarrow&\pair{\forall y\in\Gamma_{i_0}}x_0<y\:,\label{eq:strict_eli_prop_2b2}
\end{align}
where $x_0\in\Gamma_0$ in Equation \ref{eq:strict_eli_prop_2b2} is the witness to $\exists$ in the previous line. Now instantiate $y\in\Gamma_0$ in Equation \ref{eq:strict_eli_prop_2b1} to $x_0$, and instantiate $y\in\Gamma_{i_0}$ in Equation \ref{eq:strict_eli_prop_2b2} to $x_{i_0}$. Therefore, we have
\begin{align}
x_{i_0}<x_0\text{ and }x_0<x_{i_0}\:,
\end{align}
which is a contradiction. Therefore, $\pair{\exists 1\leq i\leq n}\Gamma_0\not\orde'\Gamma_i$ and $\orde'$ satisfies the second property.
\end{enumerate}
This means the strict elitist set comparison is reasonable inducing.
\end{proof}

\begin{eg}
\cite[Example 5.1]{Dung:14} Let $\lang=\set{a_i}_{i=1}^4$ be closed under (syntactic) negation, the contrary function $-$ denote symmetric negation $\neg$, $\mathcal{K}=\es$, $\relsymb_d=\set{\pair{\top\Rightarrow a_i}}_{i=1}^4$ and
\begin{align*}
\relsymb_s=\{&\pair{a_1,\:a_2,\:a_3\to\neg a_4},\:\pair{a_2,\:a_3,\:a_4\to\neg a_1},\:\\
&\pair{a_3,\:a_4,\:a_1\to\neg a_2},\:\pair{a_4,\:a_1,\:a_2\to\neg a_3}\}
\end{align*}
such that the preorder $\precsim$ is such that $d_1\approx d_2$ and $d_3\approx d_4$ only (reflexivity and transitivity is implicit). This instantiation is well-defined (Section \ref{sec:rev_ratl}, page \pageref{sec:rev_ratl}). The arguments are $A_i:=[\top\Rightarrow a_i]$ for $1\leq i\leq 4$, and
\begin{align*}
&B_4:=[A_1,\:A_2,\:A_3\to\neg a_4],\:B_1:=[A_2,\:A_3,\:A_4\to\neg a_1],\:\\
&B_2:=[A_3,\:A_4,\:A_1\to\neg a_2],\:B_3:=[A_4,\:A_1,\:A_2\to\neg a_3]\:.
\end{align*}
\noindent The strict elitist set comparison gives
\begin{align*}
&\set{d_1,\:d_2,\:d_3}\not\!\ordeneq'\set{d_4},\:\set{d_2,\:d_3,\:d_4}\not\!\ordeneq'\set{d_1},\:\\
&\set{d_3,\:d_4,\:d_1}\not\!\ordeneq'\set{d_2},\:\set{d_4,\:d_1,\:d_2}\not\!\ordeneq'\set{d_3}\:,
\end{align*}
because (e.g.) there is no defeasible rule in $\set{d_1,\:d_2,\:d_3}$ that is strictly less than $d_4$. If the witness were $d_3$, say, then $d_3\precsim d_4$, but $d_4\precsim d_3$ as well (rather than $d_4\not\precsim d_3$), so $d_3\not\prec d_4$. Therefore, under the strict elitist set comparison, we have $B_i\not\prec A_i$ (here, $\prec$ denotes the argument preference and not the preorder on defeasible rules), hence $B_i\defeat A_i$ for $1\leq i\leq 4$. The possible sets of justified arguments are $\set{A_1,\:A_2,\:A_3,\:B_4}$, $\set{A_1,\:A_2,\:B_3,\:A_4}$, $\set{A_1,\:B_2,\:A_3,\:A_4}$ and $\set{B_1,\:A_2,\:A_3,\:A_4}$, whose conclusion sets are consistent.
\end{eg}

\subsection{Summary}

We conclude that Prakken's elitist set comparison (Definition \ref{def:strict_eli_set_comparison}) should be used instead of the original elitist set comparison from \cite[page 375, Definition 19]{sanjay:13} in all all future instantiations of ASPIC$^+$, if one would like their ASPIC$^+$ instantiation to be normatively rational in the sense of \cite{Caminada:07}, and avoid counterexamples similar to that of \cite{Dung:14}.

\section{From ASPIC\texorpdfstring{$^+$}{+} to Prioritised Default Logic}\label{sec:ASPIC+_to_PDL}

\noindent We now instantiate ASPIC$^+$ to PDL, define a preference relation over arguments, and prove a representation theorem (Theorem \ref{thm:rep_thm}, page \pageref{thm:rep_thm}), which guarantees that the inferences under the argumentation semantics correspond exactly to the inferences in PDL; this is a soundness and completeness result.

\subsection{The Instantiation}\label{sec:inst}

\noindent Let $\ang{D,\:W,\:\prec^+}$ be a LPDT\footnote{We will discuss why we only consider LPDTs in Section \ref{sec:discussion_conclusions} (page \pageref{sec:discussion_conclusions}).}. The \emph{corresponding} (ASPIC$^+$) instantiation is defined as follows:
%\vspace{-0.1cm}
\begin{enumerate}
\item Our arguments are expressed in FOL, so our set of wffs is $\LForm$, although in practice we only consider $\LSent$.
%\vspace{-0.1cm}
\item The contrary function $-$ \textit{syntactically} defines conflict in terms of classical negation\footnote{For example, $\neg(\theta \wedge \neg \phi)$ is the contrary of $(\theta \wedge \neg \phi)$, but $(\theta \to \phi)$, where $\to$ in this case denotes material implication, is not the contrary of $(\theta \wedge \neg \phi)$.} so $\pair{\forall\theta\in\LForm}\:[\:\neg\theta\in\overline{\theta}$ and $\theta\in\overline{\neg \theta}\:]$.
%\vspace{-0.1cm}
\item The set of strict rules $\relsymb_s$ characterises inference in first order classical logic. We leave the specific proof theory implicit.
%\vspace{-0.1cm}
\item The set of defeasible rules $\relsymb_d$ is defined in terms of $D$ as:
\begin{align}
\relsymb_d:=\set{(\theta\Rightarrow\phi)\:\vline\:\frac{\theta:\phi}{\phi}\in D}\:,
\end{align}
with $n\equiv*$. Clearly, there is a bijection\footnote{Recall from Section \ref{sec:arg_constr_ASPIC+} (page \pageref{sec:arg_constr_ASPIC+}) that two defeasible rules are \emph{equal} iff they have the same antecedents and consequent \emph{syntactically}.} $f$ where
\begin{align}\label{eq:bij_defaults_def_rules}
f:D\to\relsymb_d:\frac{\theta:\phi}{\phi}\mapsto f\pair{\frac{\theta:\phi}{\phi}}:=\pair{\theta\Rightarrow\phi}
\end{align}
and we will define the \emph{strict version of the} preorder $\leq''$ over $\relsymb_d$ as\footnote{From Footnote \ref{fn:dual_priority_PDL} (page \pageref{fn:dual_priority_PDL}), we do not need to define $<''$ as the order-theoretic dual to $\prec^+$, avoiding potential confusion as to which item is more preferred.}
\begin{align}\label{eq:def_rules_pref_order}
(\theta\Rightarrow\phi)<''(\theta'\Rightarrow\phi')\Leftrightarrow\frac{\theta:\phi}{\phi}\prec^+\frac{\theta':\phi'}{\phi'}\:.
\end{align}
We can see that the strict toset $\ang{\relsymb_d,\:<''}$ is order isomorphic to $\ang{D,\:\prec^+}$, where the non-strict version of the order $\leq''$ abbreviates ``either $<''$ or $=$''. As we are only considering finite $D$, $\relsymb_d$ is also finite.
\vspace{-0.1cm}
\item The set of axiom premises is $\mathcal{K}_n=W$, because we take $W$ to be the set of facts. Furthermore, $\mathcal{K}_p=\es$.
\end{enumerate}

The set $\alg$ of ASPIC$^+$ arguments are defined as in Section \ref{sec:rev_ASPIC+} (page \pageref{sec:rev_ASPIC+}). It is easy to see that all arguments are firm because $\mathcal{K}_p=\es$, and so there are no undermining attacks. As $n$ is undefined, no attack can be an undercut. Therefore, we only have rebut attacks, where $A\attk B$ iff
\begin{align}\label{eq:attack}
\pair{\exists B',\:B''\subarg B}\:B'=\sqbra{B''\Rightarrow\neg Conc(A)}\:.
\end{align}

\subsection{Preferences and Defeats}

Defeats are defined as in Equation \ref{eq:ASPIC+_general_defeat} (page \pageref{eq:ASPIC+_general_defeat}). So given a suitable argument preference $\precsim$ on the arguments $\alg$ and attacks defined in the previous section, we can associate an ASPIC$^+$ defeat graph $\ang{\alg,\:\defeat}$ to any LPDT $\ang{D,\:W,\:\prec^+}$.

How should the argument preference $\precsim$ be defined based on the strict total order $<''$ over $\relsymb_d$? We would want to define $\precsim$ such that the extension of the LPDT $\ang{D,\:W,\:\prec^+}$ is given by the conclusions of the justified arguments of the defeat graph $\ang{\alg,\:\defeat}$ instantiated by the corresponding ASPIC$^+$ instantiation, and is reasonable \cite[page 372, Definition 18]{sanjay:13}.

\subsubsection{Failure of all ASPIC\texorpdfstring{$^+$}{+} Preferences}

Unfortunately, \textit{none} of the four ASPIC$^+$ argument preferences -- democratic weakest link, democratic last link, elitist weakest link, and elitist last link \cite[page page 375, Definitions 19 to 21]{sanjay:13} -- are suitable because one can devise simple LPDTs where the prioritised default extension does not correspond to the conclusions of the justified arguments. For example, for the elitist weakest link order, $\orde$:

\begin{eg}\label{eg:not_eli_WLP}
Consider the LPDT $\ang{D,\:W,\:\prec^+}$ where $W=\set{a}$,
\begin{align}
D=\set{d_1:=\frac{a:b}{b},\:d_2:=\frac{b:c}{c},\:d_3:=\frac{b:\neg c}{\neg c}}
\end{align}
and $d_1\prec^+ d_2\prec^+ d_3$. By Equation \ref{eq:ext_ind} (page \pageref{eq:ext_ind}), the prioritised default extension is $Th(\set{a,b,\neg c})$, with $d_1$ applied first, then $d_3$, which blocks $d_2$.

In the ASPIC$^+$ instantiation:  $r_1<''r_2<''r_3$ (where for $i=1,2,3$, $r_i:=f(d_i)$ and $f$ is Equation \ref{eq:bij_defaults_def_rules}). The arguments are $A:=[[[a]\Rightarrow b]\Rightarrow c]$ and $B:=[[[a]\Rightarrow b]\Rightarrow\neg c]$, which rebut each other at their conclusions.

Under the elitist ordering (Equation \ref{eq:original_non-strict_elitist}, page \pageref{eq:original_non-strict_elitist}), it is neither the case that $\set{r_1,\:r_2}\ordeneq\set{r_1,\:r_3}$ nor $\set{r_1,\:r_3}\ordeneq\set{r_1,\:r_2}$. As the sets are not equal, we have $A\not\prec B$, $B\not\prec A$ and $A\not\approx B$. This means $A\hookrightarrow B$ and $B\hookrightarrow A$ by Equation \ref{eq:ASPIC+_general_defeat}, which means there are two possible stable extensions $\set{A}$ and $\set{B}$ so that neither argument is sceptically justified, and so $\neg c$ is not an argumentation-defined inference. However $\neg c$ is a PDL inference. Therefore the elitist weakest link ordering cannot be used to calculate $\precsim$.
\end{eg}

\subsubsection{The Disjoint Elitist Order}\label{sec:disj_eli}

One can introduce the \emph{disjoint elitist order}, $\ordde$, which ignores shared rules when comparing arguments. This is intuitive because when comparing two arguments we should only focus on the fallible information on which the arguments differ. It is defined as
\begin{align}\label{eq:disj_eli}
\Gamma\orddeneq\Gamma'\Leftrightarrow\pair{\exists x\in\Gamma-\Gamma'}\pair{\forall y\in\Gamma'-\Gamma}\:x<''y\:,
\end{align}
with argument equivalence $A \approx B\Leftrightarrow DR(A)=DR(B)$ and $\ordde$ defined as usual. We call $<''$ the \textit{underlying (strict) total order} of $\orddeneq$.

If we replace Equation \ref{eq:original_non-strict_elitist} with $\orddeneq$, then from $r_2<''r_3$, it is easy to see that in Example \ref{eg:not_eli_WLP}, $A\prec B$, $B\not\prec A$, and so $A\not\hookrightarrow B$ and $B\hookrightarrow A$, and hence there is only a single stable extension containing the now sceptically justified argument $B$ with conclusion $\neg c$. This at least repairs the correspondence in Example \ref{eg:not_eli_WLP}. %, but not the correspondence in general as we shall see later.

The disjoint elitist order also satisfies a very intuitive property:

\begin{lem}
$\pair{\forall A,\:B\in\alg}\sqbra{DR(A)\subseteq DR(B)\Rightarrow B\precsim A}$. %If $A\subarg B$ then $B\precsim A$.
\end{lem}
\begin{proof}
If $DR(B)=DR(A)$ then $B\approx A$, so $B\precsim A$. If $DR(A)\subset DR(B)$, then $DR(A)-DR(B)=\es$, which means $B\prec A$ is vacuously true from Equation \ref{eq:disj_eli} so $B\precsim A$ follows.
\end{proof}

\noindent Formally, this result states that $\ordde$ extends the superset relation on $\powfin\pair{\relsymb_d}$. Intuitively, this means that the \textit{more} defeasible rules your argument contains, the \textit{less} preferred it will become. It is rational for agents to prefer more certainty than less, with all else being equal\footnote{More precisely, if the agent has two arguments, $A$ and $B$, such that $A$ is less certain than $B$ because $A$ uses more defeasible rules, then with all else being equal, the agent \textit{should} prefer $B$ over $A$. However, we do not demand that the agent must first seek complete certainty in the sense of Descartes (especially given limited knowledge and cognitive resources) prior to proposing an argument in a dialogue.}. The extreme case is that arguments with \textit{no} defeasible rules, i.e. strict arguments, are most preferred.

The above two intuitions of (1) ignoring shared elements of the sets being compared and (2) the order extends the superset relation have been considered in a different context and for a different order by \cite{Brewka:10}.

The disjoint elitist order also satisfies the following property: if $\ang{\relsymb_d,\:<''}$ is a strict toset, then $\ang{\powfin(\relsymb_d),\:\orddeneq}$ is also a strict toset. Further, $\es$ is the $\ordde$-greatest element in $\powfin(\relsymb_d)$, and $\relsymb_d$ (if finite) is the $\ordde$-least element. See Appendix \ref{app:disj_eli_props} (page \pageref{app:disj_eli_props}) for proofs of these properties.

Unfortunately, despite these desirable properties, there is a counterexample which shows the disjoint elitist order cannot be used to provide the correspondence with PDL.

\begin{eg}\label{eg:not_disj_eli_WLP}
Consider the LPDT $\ang{D,W,\prec^+}$ where $D=\set{d_1,d_2,d_3,d_4,d_5}$, $W=\es$ and
\begin{align*}
&d_1:=\frac{\top:a_1}{a_1},\:d_4:=\frac{a_3:a_4}{a_4},\:d_3:=\frac{\top:a_3}{a_3},\\
&d_2:=\frac{a_1:a_2}{a_2},\:d_5:=\frac{a_1:\neg(a_2\wedge a_4)}{\neg(a_2\wedge a_4)}\:,
\end{align*}
\noindent such that $d_1\prec^+ d_4\prec^+ d_3\prec^+ d_2\prec^+ d_5$. Our PDE is constructed in the usual manner starting from $E_0=Th(\es)$. By Equation \ref{eq:ext_ind} (page \pageref{eq:ext_ind}), the order of the application of the defaults is $d_3,d_4,d_1,d_5$, with $d_2$ blocked.
\begin{align}\label{eq:order_of_adding_rules}
E_1=E_0+a_3,\:E_2=E_1+a_4,\:E_3=E_2+a_1,\:E_4=E_3+\neg(a_2\wedge a_4)\:,
\end{align}
and $E_k=E_4$ for all $k\geq 5$. The default $d_2$ is blocked because $\neg(a_2\wedge a_4)\equiv(\neg a_2\vee\neg a_4)$, and with $a_4$ (from $d_4$), we have $\neg a_2$, which blocks $d_2$. The unique PDE from this LPDT is
\begin{align}\label{eq:PDE}
E=Th\pair{\set{a_1,a_3,a_4,\neg(a_2\wedge a_4)}}=Th(\set{a_1,\neg a_2,a_3,a_4})\:.
\end{align}

Now consider the corresponding arguments from our instantiation. We have the defeasible rules\footnote{Where, similar to Example \ref{eg:not_eli_WLP}, $r_i$ corresponds to $d_i$ via Equation \ref{eq:bij_defaults_def_rules} (page \pageref{eq:bij_defaults_def_rules}).}
\begin{align}\label{eq:not_disj_eli_rules}
r_1<''r_4<''r_3<''r_2<''r_5\:.
\end{align}
The relevant arguments and sets of defeasible rules are
\begin{align}\label{eq:def_rules_ord}
A&:=[[\top\Rightarrow a_1]\Rightarrow a_2]\:,DR(A)=\set{r_1,r_2}\\
B&:=[[\top\Rightarrow a_3]\Rightarrow a_4]\:,DR(B)=\set{r_3,r_4}\\
C&:=[[\top\Rightarrow a_1]\Rightarrow\neg(a_2\wedge a_4)]\:,DR(C)=\set{r_1,r_5}\:,\\
D&:=[B,C\to\neg a_2]\:,DR(D)=\set{r_1,r_3,r_4,r_5}\:,
\end{align}

\noindent We illustrate these arguments in Figure \ref{figure:not_disj_eli_WLP}.

\begin{figure}[ht]
\begin{center}
\includegraphics[height=3.9cm,width=4cm]{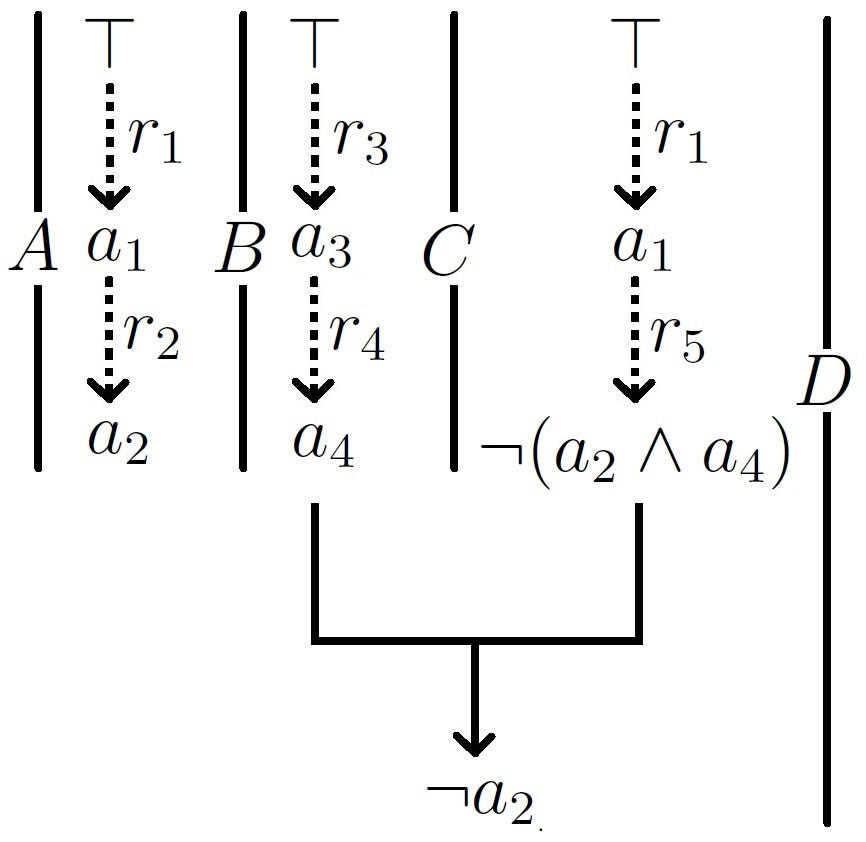}
\caption{The arguments in Example \ref{eg:not_disj_eli_WLP}. The dashed arrows denote defeasible rules and the solid arrows denote strict rules. Arguments $A$, $B$ and $C$ are clearly indicated by the label on the line to the left of the constituted argument. Argument $D$ is constructed from arguments $B$ and $C$, and the strict rule $\pair{a_4,\:\neg\pair{a_2\wedge a_4}\to\neg a_2}\in\relsymb_s$.}
\label{figure:not_disj_eli_WLP}
\end{center}
\end{figure}

The stable extension is $\set{D,\:B,\:C,\:[\top\Rightarrow a_3],\:[\top\Rightarrow a_1]}$ and all strict extensions thereof\footnote{Informally, in ASPIC$^+$, for $S\subseteq\alg$ the strict extension of $S$ is the smallest set containing $S$ extended with all strict and firm arguments, and all possible applications of strict rules to those arguments. This becomes the deductive closure when ASPIC$^+$ is instantiated into classical logic. See \cite[page 370, Definition 17]{sanjay:13} for more details.}. The conclusion set does correspond to Equation \ref{eq:PDE}. However, this would require $D\defeat A$, which means, by Equation \ref{eq:ASPIC+_general_defeat} (page \pageref{eq:ASPIC+_general_defeat}), $D\attk A$ and $D\not\prec A$. Clearly, $D\attk A$ on $A$. However, $D\not\prec A$ means, under the disjoint elitist order,
\begin{align}
D\not\prec A&\Leftrightarrow\text{not }DR(D)\orddeneq DR(A)\nonumber\\
&\Leftrightarrow\text{not }\pair{\exists x\in DR(D)-DR(A)}\pair{\forall y\in DR(A)-DR(D)}x<'' y\nonumber\\
&\Leftrightarrow\text{not }\pair{\exists x\in\set{r_1,r_3,r_4,r_5}}\pair{\forall y\in\set{r_2}}x<''y\nonumber\\
&\Leftrightarrow\text{not }\pair{\exists x\in\set{r_1,r_3,r_4,r_5}x<''r_2}\nonumber\\
&\Leftrightarrow\text{not }\pair{r_1<''r_2\text{ or }r_3<''r_2\text{ or }r_4<''r_2\text{ or }r_5<''r_2}\nonumber\\
&\Leftrightarrow\:r_2<''r_1,\:r_2<''r_3,\:r_2<''r_4,\:r_2<''r_5\:.\label{eq:violation_of_reas}
\end{align}

\noindent From Equation \ref{eq:not_disj_eli_rules}, it is not the case that $r_2<''r_1,r_3,r_4$, so we conclude $D\prec A$. Therefore, argumentation does not generate the corresponding stable extension to Equation \ref{eq:PDE}.

%It is instructive to work out the stable extensions in this case using the disjoint elitist order. Define the arguments
%\begin{align}
%E:=[C,A\to\neg a_4],\:F:=[A,B\to(a_2\wedge a_4)]\:.
%\end{align}
%We can see that $D\attk A$ and hence $D\attk E,\:F$, $E\attk B$ hence $E\attk F,\:D$, and $F\attk C$ hence $F\attk D,\:E$. However, we can show that $F\prec C$, $D\prec A$ and $E\prec B$. Therefore, none of the above attacks succeed as defeats, and hence the stable extension is, prior to subargument closure, $\set{A,B,C}$. However, the stable extension has to be indirectly consistent \cite{Caminada:07}, which means the deductive closure (closure under strict rules) of the conclusions is consistent, and this is false, because $B$ and $C$ can be extended to $D$, which attacks $A$.
\end{eg}

\subsubsection{Mimicking Prioritised Default Logic}\label{sec:SP_order}

Despite not being suitable for our desired correspondence between ASPIC$^+$ and PDL, the disjoint elitist order does capture one important intuition. When comparing two arguments $A,\:B\in\alg$, we compare them at their defeasible rules, and whichever argument has the $<''$-least rule in the set $DR(A)\ominus DR(B)$ is the less preferred argument. However, comparing sets of defeasible rules does not take the \textit{structure} of arguments into account, i.e. in terms of which rules in the construction of arguments could be applied earlier, and which could be applied later. We now transform $<''$ into a new order, $<_{SP}$, called the \textit{structure preference} order, such that it gives the correct argument preference for the correspondence.

Let $\ang{\relsymb_d,<''}$ be given. Given $R\subseteq\relsymb_d$, we define the set $Args(R)\subseteq\alg$ to be the set of arguments such that $A\in Args(R)\Leftrightarrow DR(A)\subseteq R$. We call this set \textit{the set of arguments freely constructed with defeasible rules in $R$}. It can be easily shown that $Args$ is $\subseteq$-monotonic in $R$ and that the assignment $R\mapsto Args(R)$ is functional. Clearly, $Args(\es)$ is the set of all strict arguments, and $Args\pair{\relsymb_d}=\alg$. Further, $Args(R)$ is \textit{subargument-closed}, i.e.

\begin{cor}\label{cor:args_op_is_subarg_closed}
If $A\in Args(R)$ and $B\subarg A$, then $B\in Args(R)$, for any $R\subseteq\relsymb_d$.
\end{cor}
\begin{proof}
It is easy to show that if $A\subarg B$, then $DR(A)\subseteq DR(B)$. Therefore, if $A\in Args(R)$, then $DR(A)\subseteq R$, and hence $DR(B)\subseteq R$. Therefore, $B\in Args(R)$.
\end{proof}

For $R\subseteq\relsymb_d$ let $\max_{<''}R\subseteq R$ denote the \textit{set} of all $<''$-maximal elements of $R$. As $<''$ is a (strict) total order and $\relsymb_d$ is finite, this is a singleton set. For $r\in\relsymb_d$ recall the $Ante$ map (Equation \ref{eq:antecedent_map_for_rules}, page \pageref{eq:antecedent_map_for_rules}). Note that we are considering defeasible rules with one antecedent, so $Ante(r)$ is a singleton set. For $S\subseteq\alg$ recall the $Conc$ map (Equation \ref{eq:set_arg_attrb}, page \pageref{eq:set_arg_attrb}).

Consider ordering the rules in $\relsymb_d$ as follows: for $1\leq i\leq |\relsymb_d|$, we define the singleton set $\set{a_i}\subseteq\relsymb_d$ to be
\begin{align}\label{eq:SP_ord_def}
\max_{<''}\pair{\set{r\in\relsymb_d\:\vline\:Ante(r)\subseteq Conc\pair{Args\pair{\bigcup_{k=1}^{i-1}\set{a_k}}}}-\bigcup_{j=1}^{i-1}\set{a_j}}\:.
\end{align}

\noindent More concretely,
\begin{align*}
\set{a_1}=&\max_{<''}\set{r\in\relsymb_d\:\vline\:Ante(r)\subseteq Conc\pair{Args\pair{\es}}}\:,\\
\set{a_2}=&\max_{<''}\pair{\set{r\in\relsymb_d\:\vline\:Ante(r)\subseteq Conc\pair{Args\pair{\set{a_1}}}}-\set{a_1}}\:,\\
\set{a_3}=&\max_{<''}\pair{\set{r\in\relsymb_d\:\vline\:Ante(r)\subseteq Conc\pair{Args\pair{\set{a_1,\:a_2}}}}-\set{a_1,\:a_2}}\:,
\end{align*}
\noindent and so on, until this process stops at $\set{a_{|\relsymb_d|}}$\footnote{This is just Equation \ref{eq:SP_ord_def} with $i=|\relsymb_d|$, which is where the assumption that $\relsymb_d$ is a finite set is crucial.}. The intuition is: $a_1$ is the most preferred rule whose antecedent is amongst the conclusions of all strict (and firm) arguments, $a_2$ is the next most preferred rule, whose antecedent is amongst the conclusions of all arguments having at most $a_1$ as a defeasible rule. Similarly, $a_3$ is the next most preferred rule, whose antecedent is amongst the conclusions of all arguments having at most $a_1$ and $a_2$ as defeasible rules, and so on until all of the rules of $\relsymb_d$ are exhausted. This process orders the rules by how preferred they are under $<''$ \textit{and} by how earlier they are applicable when constructing the arguments.

We then define (notice the dual order)
\begin{align}\label{eq:SP_order}
a_i <_{SP} a_j\Leftrightarrow j < i\:,
\end{align}
where $1\leq i,\:j\leq|\relsymb_d|$. We define the non-strict order to be $a_i\leq_{SP}a_i\Leftrightarrow\sqbra{a_i=a_j\text{ or }a_i<_{SP}a_j}$. This makes sense because $i\mapsto a_i$ is bijective between $\relsymb_d$ and $\set{1,2,3,\ldots,|\relsymb_d|}$. Clearly $<_{SP}$ is a strict total order on $\relsymb_d$. We call this the \textit{structure preference order} on $\relsymb_d$, which exists and is unique given $<''$. The corresponding argument preference, $\prec_{SP}$, is $<_{SP}$ under the disjoint elitist order\footnote{We use the disjoint elitist order instead of the usual elitist order because Example \ref{eg:not_eli_WLP} (page \pageref{eg:not_eli_WLP}) shows that the usual elitist order does not give the correspondence.}, i.e.
\begin{align}\label{eq:SP_arg_pref}
A\prec_{SP} B\Leftrightarrow\pair{\exists x\in DR(A)-DR(B)}\pair{\forall y\in DR(B)-DR(A)}\:x<_{SP} y\:,
\end{align}
with $\precsim_{SP}$ defined as usual. We can also define the corresponding set comparison relation, $\ordneq_{SP}$, as, for $\Gamma,\:\Gamma'\subseteq_\text{fin}\relsymb_d$,
\begin{align}
\Gamma\ordneq_{SP}\Gamma'\Leftrightarrow\pair{\exists x\in\Gamma-\Gamma'}\pair{\forall y\in\Gamma'-\Gamma}x<_{SP}y\:,
\end{align}
such that $\Gamma\ord_{SP}\Gamma'\Leftrightarrow\sqbra{\Gamma\ordneq_{SP}\Gamma'\text{ or }\Gamma=\Gamma'}$, and
\begin{align}
A\precsim_{SP}B\Leftrightarrow DR(A)\ord_{SP} DR(B)\:.
\end{align}
The name ``structure preference order'' refers to the fact that this order takes into account both the preference $<''$ and the structure, i.e. when the rule is first applicable during the construction of arguments. This allows us to imitate how PDL applies defaults when calculating extensions.

\begin{eg}
(Example \ref{eg:not_disj_eli_WLP}, page \pageref{eg:not_disj_eli_WLP} continued) We have the following: %(recalling $Args(R)$ is subargument closed for any $R\subseteq\relsymb_d$, Corollary \ref{cor:args_op_is_subarg_closed}, page \pageref{cor:args_op_is_subarg_closed}):
\begin{align*}
\set{a_1}=&\max_{<''}\pair{\set{r\in\relsymb_d\:\vline\:Ante(r)\subseteq Conc\pair{Args\pair{\es}}}-\es}\\
=&\max_{<''}\set{r\in\relsymb_d\:\vline\:Ante(r)\subseteq\set{\top}}\\
=&\max_{<''}\set{r_1,r_3}=\set{r_3}\implies a_1=r_3\:.
\end{align*}
\begin{align*}
\set{a_2}=&\max_{<''}\pair{\set{r\in\relsymb_d\:\vline\:Ante(r)\subseteq Conc\pair{Args\pair{\set{r_3}}}}-\set{r_3}}\\
=&\max_{<''}\pair{\set{r_1,r_3,r_4}-\set{r_3}}\\
=&\max_{<''}\set{r_1,r_4}=\set{r_4}\implies a_2=r_4\:.
\end{align*}
\begin{align*}
\set{a_3}=&\max_{<''}\pair{\set{r\in\relsymb_d\:\vline\:Ante(r)\subseteq Conc\pair{Args\pair{\set{r_3,r_4}}}}-\set{r_3,r_4}}\\
=&\max_{<''}\pair{\set{r_1,r_3,r_4}-\set{r_3,r_4}}\\
=&\max_{<''}\set{r_1}\implies a_3=r_1\:.
\end{align*}
\begin{align*}
\set{a_4}=&\max_{<''}\pair{\set{r\in\relsymb_d\:\vline\:Ante(r)\subseteq Conc\pair{Args\pair{\set{r_3,r_4,r_1}}}}-\set{r_1,r_3,r_4}}\\
=&\max_{<''}\pair{\set{r_2,r_5,r_1,r_3,r_4}-\set{r_1,r_3,r_4}}\\
=&\max_{<''}\set{r_2,r_5}=\set{r_5}\implies a_4=r_5\:.\\
\set{a_5}=&\set{r_2}\implies a_5=r_2\:.
\end{align*}
Therefore, we have
\begin{align}
a_1=r_3,\:a_2=r_4,\:a_3=r_1,\:a_4=r_5,\:a_5=r_2\:.
\end{align}
The structure preference order is
\begin{align}
r_2<_{SP}r_5<_{SP}r_1<_{SP}r_4<_{SP}r_3\:.
\end{align}
Notice that this is precisely the order of how the corresponding normal defaults are added in PDL, as Equation \ref{eq:order_of_adding_rules} (page \pageref{eq:order_of_adding_rules}) shows. It is easy to show that the corresponding stable extension under $\prec_{SP}$ corresponds to the PDL inference, because $r_2$ is now $<_{SP}$-least, so $D\not\prec_{SP}A$ by Equation \ref{eq:violation_of_reas} (page \pageref{eq:violation_of_reas}).
\end{eg}

However, $<_{SP}$ does not necessarily follow the PDL order of applying defaults as the following example illustrates.

\begin{eg}\label{eg:always_blocked_default}
Consider the LPDT $\ang{D,\:W,\:\prec^+}$ where $W=\set{a}$, $d_1:=\frac{a:\neg a}{\neg a}$ and $d_2:=\frac{\top:b}{b}$, such that $d_2\prec^+ d_1$. The prioritised default extension is $E=Th\pair{\set{a,b}}$, where $d_1$ is blocked by $W$, so $d_2$ is the only default added.

Translating this to argumentation, we have $\mathcal{K}_n=\set{a}$, $r_1:=(a\Rightarrow\neg a)$ and $r_2:=(\top\Rightarrow b)$ where for $i=1,2$, $r_i=f\pair{d_i}$, such that $r_2<'' r_1$. The arguments are $A_0:=[a]$, $A_1:=[A_0\Rightarrow\neg a]$ and $B:=[\top\Rightarrow b]$. Applying Equation \ref{eq:SP_ord_def} (page \pageref{eq:SP_ord_def}), we have $r_2<_{SP}r_1$, which clearly is not the order of how the corresponding defaults are added in PDL.

Yet the correspondence still holds. Clearly $A_0\defeat A_1$ because $A_0$ is strict, so the stable extension is the strict extension of $\set{A_0,\:B}$, the conclusion set of which is the extension from PDL.
\end{eg}

Example \ref{eg:always_blocked_default} highlights how blocked defaults and defeated arguments are related. Where PDL blocks the application of a given default and hence preventing its conclusion from featuring in the extension, ASPIC$^+$ allows for the construction of the argument with the corresponding defeasible rule, but that argument is always defeated by another strictly stronger argument and therefore cannot be in any extension.

\subsection{Correspondence of Inferences}

In this section we prove that the argument preference $\prec_{SP}$ is the suitable order to give a correspondence between PDL and ASPIC$^+$ in all cases. Given an LPDT $\ang{D,\:W,\:\prec^+}$, we can construct its defeat graph $\ang{\alg,\:\defeat}$ where the ASPIC$^+$ arguments $\alg$ are constructed following Section \ref{sec:inst} (page \pageref{sec:inst}), the attacks $\attk$ are rebuts at the conclusions of defeasible rules, and the defeats are as in Equation \ref{eq:ASPIC+_general_defeat} (page \pageref{eq:ASPIC+_general_defeat}) with $\attk$ and $\precsim_{SP}$. It is always possible to construct $<_{SP}$ and hence $\precsim_{SP}$ after translating the LPDT to ASPIC$^+$.

\subsubsection{Uniqueness of Stable Extensions}

\noindent In this section we show that the ASPIC$^+$ defeat graphs $\ang{\alg,\:\defeat}$ that have been constructed from a LPDT $\ang{D,\:W,\:\prec^+}$ (following Section \ref{sec:inst}, page \pageref{sec:inst}) each has a unique stable extension. For any starting preference $<''$ on $\relsymb_d$ we first construct $<_{SP}$ following Equations \ref{eq:SP_ord_def} and \ref{eq:SP_order} (page \pageref{eq:SP_order}).

\begin{thm}\label{thm:total_still_has_unique_stable_extension}
Let $\ang{\alg,\:\rightharpoonup,\:\precsim_{SP}}$ be an ASPIC$^+$ attack graph constructed from $\lang=\LForm$, $-$ is $\neg$, $\relsymb_s$ the rules of proof of FOL, $\ang{\relsymb_d,\:<''}$ a \textit{finite} strict \emph{toset} of defeasible rules, $\precsim_{SP}$ is $<_{SP}$ under the disjoint elitist order, $n\equiv *$ on $\relsymb_d$, $\mathcal{K}_p=\es$ and $\mathcal{K}_n\subseteq\LForm$ is a consistent set of formulae. The defeat graph $\ang{\alg,\:\defeat}$ from this attack graph has a unique stable extension.
\end{thm}

\begin{proof}
The construction of the unique stable extension mimics how extensions are constructed over an LPDT (Equation \ref{eq:ext_ind}, page \pageref{eq:ext_ind}). Given a set of arguments $S\subseteq\alg$ we define, for $r\in\relsymb_d$, $S\oplus r:=Args(DR(S)\cup\set{r})$, i.e. we close $S$ under all arguments with the addition of a new defeasible rule $r$. Now consider Algorithm \ref{alg:gen_stab_ext}. We input the ASPIC$^+$ attack graph $\ang{\alg,\:\attk,\:\precsim_{SP}}$ as described by the hypothesis of the theorem, and the algorithm outputs a set of arguments.

\begin{algorithm}[ht]
\begin{algorithmic}[1]
\Function{GenerateStableExtension}{$\ang{\alg,\:\attk,\:\precsim_{SP}}$}
  \State $S\gets Args(\es)$\label{alg_line:input_strict_args}
  \For{$r\in\relsymb_d$ from $<_{SP}$-greatest to $<_{SP}$-smallest}
    \If{$S\oplus r$ is attack-cf}\label{alg_line:cond_start}
    \State {$S\gets S\oplus r$\label{alg_line:cond_end}}
    \EndIf
  \EndFor
  \Return $S$
\EndFunction
\end{algorithmic}
\caption{Generating a Stable Extension}
\label{alg:gen_stab_ext}
\end{algorithm}

The intuition of this algorithm is to create the largest possible set of undefeated arguments, first by including all strict arguments because strict arguments are never defeated (Line \ref{alg_line:input_strict_args}) and never attack each other because $\mathcal{K}_n$ is consistent. Then, the algorithm includes the defeasible rules from most to least preferred and tests whether the resulting arguments that are constructed by the inclusion of such a defeasible rule leads to an attack (Lines \ref{alg_line:cond_start}--\ref{alg_line:cond_end}). As $<_{SP}$ is total, all defeasible rules are considered. This algorithm halts because $\relsymb_d$ is finite.

It is clear from the algorithm that $S$ exists and is unique, because $S$ is a set of freely-constructed arguments (i.e. of the form $Args(R)$ for some $R\subseteq\relsymb_d$) including as many mutually compatible defeasible rules as possible. We now show that $S$ is a stable extension \cite[page 26 Definition 2.2.7]{EoA}.

\textit{Attack-cf:} This is guaranteed by the consistency of $\mathcal{K}_n$, so two strict arguments cannot attack each other, and that defeasible rules $r\in\relsymb_d$ are only added if attack-cf is preserved (Lines \ref{alg_line:cond_start}--\ref{alg_line:cond_end}). Therefore, $S$ must be attack-cf.

\textit{Defeats all other arguments:} Let $R\subseteq\relsymb_d$ be the set of all defeasible rules added to $S$, i.e.
\begin{align}
R:=DR(S)=\bigcup_{A\in S}\:DR(A)\subseteq\relsymb_d\:.
\end{align}

Let $B\notin S$ be any argument. As $S=Args\pair{DR(S)}$, this means there is some rule $r\in DR(B)$ such that $r\notin R$. The only reason why $r\notin R$ is because if $r$ were added to $R$, then the resulting $S$ would not be attack-cf, according to Algorithm \ref{alg:gen_stab_ext}. Let $B'\subarg B$ be such that $TopRule(B')=r$, which must exist by the inductive construction of arguments. Let $A$ be the attacker of $r$, such that\footnote{Note that $A$ is appropriately chosen such that $Conc(A)=\neg Cons(r)$ is syntactic equality. This is always possible because $\relsymb_s$ has all rules of proof of FOL. Therefore, if an argument $C$ concludes $\theta$, and we would want it to conclude $\phi$, where $\phi$ and $\theta$ are logically equivalent, we can just append the strict rule $(\theta\to\phi)\in\relsymb_s$ to $C$ to create a new argument $D$ that concludes $\phi$.} $Conc(A)=\neg Cons(r)$ (Equation \ref{eq:consequent_map_for_rules}, page \pageref{eq:consequent_map_for_rules}), so $A\attk B'$ and hence $A\attk B$. There are two possibilities: either $r$ is $<_{SP}$-greatest or it is not.

Suppose $r$ is $<_{SP}$-greatest, then $Args\pair{\es}\oplus r$ is not attack-cf, so $A\in Args(\es)$. As $A$ is strict, $A\defeat B'$, and hence $A\defeat B$.

Now suppose $r$ is not $<_{SP}$-greatest. Consider the strict up-set of $r$ in $\relsymb_d$,
\begin{align}
T:=\set{r'\in\relsymb_d\:\vline\:r<_{SP}r'}\neq\es\:.
\end{align}
There are two sub-possibilities: either $T\cap R=\es$ or $T\cap R\neq\es$. If the former, then given that adding $r$ to $S$ will create an attack from $A\in Args(\es)$, we have $A\defeat B$. If the latter, i.e. $T\cap R\neq\es$, adding $r$ to $S$ means its attacker $A\attk B'$ is in $Args\pair{T\cap R}$. Either $A$ is strict or not strict (i.e. defeasible). If it is strict, then $A\defeat B$ as before. If it is not strict, i.e. $\es\neq DR(A)\subseteq T\cap R$, then by definition $\pair{\forall s\in T}r<_{SP}s$. As $DR(A)\subseteq T\cap R\subseteq T$, we must also have $\pair{\forall s\in DR(A)}r<_{SP} s$. Therefore, there is an $r\in DR(B')-DR(A)$ such that for all rules in $DR(A)$, and hence $DR(A)-DR(B')$, $r<_{SP} s$. By Equation \ref{eq:SP_arg_pref} (page \pageref{eq:SP_arg_pref}), we conclude that $B'\prec_{SP} A$, and hence $A\defeat B'$. Therefore, by definition of $\defeat$ and $\subarg$, $A\defeat B$.

We conclude that the defeat graphs of these ASPIC$^+$ attack graphs have a unique stable extension.
\end{proof}

\subsubsection{A Helpful Distinction in Prioritised Default Logic}

In this section we formalise a distinction between defaults in PDL that are blocked because there exists something that disagrees with them, and defaults that are blocked because they do not add any new information.

Let $\ang{D,\:W,\:\prec}$ be a PDT and $E=\bigcup_{i\in\nat}E_i$ one of its extensions generated from the linearisation $\prec^+\supseteq\prec$. The \emph{set of generating defaults (with respect to $\prec^+$), $GD(\prec^+)$}, is defined as
\begin{align}\label{eq:GD}
GD_i(\prec^+)&:=\set{d\in D\:\vline\:\text{$d$ is $\prec^+$-greatest active in $E_i$}}\:,\nonumber\\
GD(\prec^+)&:=\bigcup_{i\in\nat}GD_i(\prec^+)\subseteq D\:.
\end{align}
Intuitively, this is the set of defaults applied to calculate $E$ following the order $\prec^+$. However, the same $E$ can be generated by distinct total orders $\prec^+$.
\begin{eg}
Consider the PDT $\ang{\set{\frac{a:c}{c},\:\frac{b:c}{c}},\:\set{a,\:b},\:\es}$. We have two linearisations $\frac{a:c}{c}\prec_1^+\frac{b:c}{c}$ and $\frac{b:c}{c}\prec_2^+\frac{a:c}{c}$. We have $GD(\prec_1^+)=\set{\frac{a:c}{c}}$ and $GD(\prec_2^+)=\set{\frac{b:c}{c}}$, which are not equal, even though both linearisations give the same extension $E=Th\pair{\set{a,\:b,\:c}}$. But in the case of $\prec_1^+$, $\frac{b:c}{c}$ is not active because it adds no new information, rather than that we know $\neg c$ already.
\end{eg}
We wish to distinguish between inactive defaults that conflict with something we already know, and inactive defaults that add no new information. We call a default $\frac{\theta:\phi}{\phi}$ \emph{semi-active (in $S\subseteq\LSent$)} iff $\sqbra{\theta\in S,\:\neg\phi\notin S,\:\phi\in S}$. The \emph{set of semi-active defaults (with respect to the linearisation $\prec^+$)} is
\begin{align}\label{eq:SAD}
SAD(\prec^+):=\set{d\in D\:\vline\:\text{ $d$ is semi-active w.r.t. $\prec^+$}}\:.
\end{align}
Intuitively, the application of semi-active defaults add no new information. We then define the \emph{set of non-blocked defaults} to be
\begin{align}\label{eq:NBD}
NBD(\prec^+):=GD(\prec^+)\cup SAD(\prec^+)\subseteq D\:.
\end{align}
Intuitively, this is precisely the set of defaults which are not blocked by the information that has accumulated in the previous steps. This includes the defaults that we have used to nonmonotonically infer some knowledge, \textit{and} the defaults that do not add any new information.

The set of non-blocked defaults has a more elegant characterisation:
\begin{lem}\label{lem:NBD_characterisation}
If $\prec^+$ generates $E$, then we have that
\begin{align}\label{eq:lem_NBD}
NBD(\prec^+):=\set{\frac{\theta:\phi}{\phi}\in D\:\vline\:\theta\in E,\:\neg\phi\notin E}\:.
\end{align}
\end{lem}
\begin{proof}
See Appendix \ref{app:NBD_char} (page \pageref{app:NBD_char}).
\end{proof}

\noindent Equation \ref{eq:lem_NBD} shows that $NBD$ exists and is unique given an extension $E$. We may write $NBD(E)$ instead of $NBD\pair{\prec^+}$, or just $NBD$ when $E$ is clear from context. Equation \ref{eq:lem_NBD} adapts Reiter's idea of a \emph{generating default} \cite[page 92 Definition 2]{Reiter:80} to PDL. The set $NBD(E)$ can always be calculated in PDL once $E$ is obtained.

\subsubsection{The Representation Theorem}

\noindent In this section we state and prove the representation theorem, which guarantees that the inferences under the argumentation semantics correspond exactly to the inferences in PDL under $\precsim_{SP}$; this is a soundness and completeness result. More specifically, the theorem relates the (unique) stable extension of $\ang{\alg,\:\defeat}$ with the (unique) prioritised default extension of the corresponding LPDT $\ang{D,\:W,\:\prec^+}$.

\begin{thm}\label{thm:rep_thm}
Let $\ang{\alg,\:\attk,\:\precsim_{SP}}$ be the attack graph corresponding\footnote{Recall that we transform $\prec^+$ to $<''$ using Equation \ref{eq:def_rules_pref_order} (page \pageref{eq:def_rules_pref_order}), and then apply Equation \ref{eq:SP_order} (page \pageref{eq:SP_order}) to obtain $<_{SP}$, which gives $\precsim_{SP}$ as in Equation \ref{eq:SP_arg_pref}.} to an LPDT $\ang{D,\:W,\:\prec^+}$, with defeat graph $\ang{\alg,\:\defeat}$.
\begin{enumerate}
\item Let $E$ be the prioritised default extension of $\ang{D,\:W,\:\prec^+}$, $NBD(E)$ be the set of non-blocked defaults (Equation \ref{eq:lem_NBD}, page \pageref{eq:lem_NBD}) and $R:=f\pair{NBD(E)}$ be the (image) set of corresponding defeasible rules (where $f$ is Equation \ref{eq:bij_defaults_def_rules}, page \pageref{eq:bij_defaults_def_rules}), then $Args\pair{R}$ is the stable extension of $\ang{\alg,\:\defeat}$.
\item Let $\ext\subseteq\alg$ be the unique stable extension of $\ang{\alg,\:\defeat}$ by Theorem \ref{thm:total_still_has_unique_stable_extension} (page \pageref{thm:total_still_has_unique_stable_extension}), then $Conc\pair{\ext}$ (Equation \ref{eq:set_arg_attrb}, page \pageref{eq:set_arg_attrb}) is the prioritised default extension of $\ang{D,\:W,\:\prec^+}$.
\end{enumerate}
\end{thm}
\begin{proof}
We will prove each part separately. For the first statement we will show that the given $E$ is the extension generated from $\prec^+$, and $R=f\pair{NBD(E)}$ is the set of defeasible rules corresponding to the defaults used in $E$ together with the semi-active defaults, the set of arguments $Args(R)$ is a stable extension. For the second statement, we let $E$ be the prioritised default extension of $\ang{D,\:W,\:\prec^+}$ and show $Conc(\ext)\subseteq E$ and $E\subseteq Conc(\ext)$.\\

\noindent\textit{1. $Args(R)$ is a stable extension:}\\

\noindent To show that $Args(R)$ is a stable extension, it is sufficient to show $Args(R)$ is attack-cf and that for all arguments $B\notin Args(R)$, there is some argument $A\in Args(R)$ such that $A\defeat B$ \cite[page 26 Definition 2.2.7]{EoA}.\\

\noindent\textit{1.1 $Args(R)$ is attack-cf:}\\

\noindent To show that $Args(R)$ is attack-cf we have to show that no two arguments $A,\:B\in Args(R)$ attack each other. Assume for contradiction that $Args(R)$ is not attack-cf, then $\pair{\exists A,\:B\in Args(R)}A\attk B$. By definition, $DR(B)\subseteq R=f\pair{NBD(E)}$. Suppose $A\attk B$ on $B'\subarg B$, so $\pair{\exists B''\subarg B}\:B'=[B''\Rightarrow\neg Conc(A)]$ by Equation \ref{eq:attack} (page \pageref{eq:attack}). Let $r=TopRule(B')$, so $r=\pair{Conc(B'')\Rightarrow\neg Conc(A)}$. Clearly $r\in DR(B)$, and hence $f^{-1}(r)=\frac{Conc(B''):\neg Conc(A)}{\neg Conc(A)}\in NBD(E)$, and $\neg Conc(A)\in E$.

However, we also have $DR(A)\subseteq R$ as well. Let $\set{r_1,\:\cdots,\:r_n}\subseteq DR(A)$ be the set of defeasible rules such that $W\cup \set{Cons(r_1),\:\cdots,\:Cons(r_n)}\models\neg Conc(A)$ ($Cons$ is defined in Equation \ref{eq:consequent_map_for_rules}, page \pageref{eq:consequent_map_for_rules}). For $1\leq i\leq n$ let $d_i:=f^{-1}\pair{r_i}$. Clearly the corresponding defaults $d_1,\:\ldots,\:d_n\in NBE(E)$ and hence $Cons(r_i)\in E$ for $1\leq i\leq n$. As $E$ is deductively closed and $W\subseteq E$, then $\neg Conc(A)\in E$. Therefore, $E$ is inconsistent - contradiction, because $W$ is consistent. Therefore, $Args(R)$ is attack-cf.\\

\noindent\textit{1.2 $Args(R)$ defeats all other arguments:}\\

\noindent Now we show that $Args(R)$ defeats all other arguments, by showing that for every argument $B\notin Args(R)$ there exists an argument in $A\in Args(R)$ such that $A\hookrightarrow B$. Let $B\notin Args(R)$ be arbitrary, which means there is some $r\in DR(B)$ such that $r\notin R$. Let $B'\subarg B$ be such that $TopRule(B')=r$. Let $r=(\theta\Rightarrow\phi)$, so $r\notin R$ means $f^{-1}\pair{r}=\frac{\theta:\phi}{\phi}\notin NBD(E)$. By Equation \ref{eq:lem_NBD} (page \pageref{eq:lem_NBD}), this means $\theta\notin E$ or $\neg\phi\in E$. This gives us two possibilities: either $\theta\notin E$, or $\neg\phi\in E$.\\

\noindent\textit{1.2.1 The case of $\neg\phi\in E$:}\\

\noindent Assume $\neg\phi\in E$, then $\pair{\exists i\in\nat}\:\neg\phi\in E_i$ by Equations \ref{eq:ext_base} and \ref{eq:ext_ind} (page \pageref{eq:ext_base}). Either $i=0$ or $i>0$.\\

\noindent\textit{1.2.1.1 The case of $i=0$:}\\

\noindent Suppose $i=0$, then $W\models\neg\phi$ from Equation \ref{eq:ext_base}. By compactness, there is some finite $W'\subseteq W$ such that $W'\models\neg\phi$. We can construct an argument $A$ such that $Prem(A)=W'$ and $Conc(A)=\neg\phi$ as there will be appropriate combinations of strict rules in $\relsymb_s$, so $A\rightharpoonup B$. As $DR(A)=\es\subseteq R$, we must have $A\in Args(R)$. Further, as $A$ is strict, $A\hookrightarrow B$ is guaranteed\footnote{This is because $\precsim_{SP}$ is based on the disjoint elitist order, which ranks $\es$ as the greatest element in $\powfin\pair{\relsymb_d}$.} by $\precsim_{SP}$.\\

\noindent\textit{1.2.1.2 The case of $i>0$:}\\

\noindent Now suppose $i>0$, then $\neg\phi\in E_j$ where $j>0$ is the witness for $i$. Let $d_j\in D$ be the default that is $\prec^+$-greatest active in the layer $E_j$, so the set of defaults that conclude $\neg\theta$ (up to the application of deductive rules) is $S:=\set{d_0,\:\ldots,\:d_{j-1}}\subseteq GD_{j-1}\pair{\prec^+}\subseteq NBD(E)$. We can construct an argument $A$ such that $Prem(A)\subseteq W$, $Conc(A)=\neg\phi$ and $DR(A)=f(S)$. Clearly, $DR(A)=f(S)\subseteq f\pair{NBD(E)}=R$ and hence $A\in Args\pair{R}$. It is clear that $A\attk B$, so we need to show $B\not\prec_{SP} A$.

Given that $\neg\phi\in E_j$, it must be the case that $\phi\notin E_j$. Therefore, $r$ is not $\prec^+$-greatest active for all extension layers $E_0,\:\ldots,\:E_{j-1}$. Suppose for contradiction that there are some rules $s\in DR(A)$ that are $<_{SP}$-smaller than $r$. Then by Equation \ref{eq:SP_ord_def} (page \pageref{eq:SP_ord_def}), $r$ must be $\prec^+$-greatest active at some $E_k$ for $k<j-1$, which would then result in $\phi$ in $E_{k+1}$, therefore preventing $\neg\phi\in E_j$ - contradiction. Therefore, $r$ is $<_{SP}$-smaller than all rules in $DR(A)$. By Equation \ref{eq:SP_arg_pref}, we must have $B\prec_{SP} A$, and hence $A\not\prec_{SP} B$, so $A\defeat B$. Therefore, for the case of $\neg\phi\in E$, $Args(R)$ defeats all arguments outside it.\\

\noindent\textit{1.2.2 The case of $\theta\notin E$:}\\

\noindent Now assume $\theta\notin E$. We will show this case is impossible by using the method of infinite descent\footnote{That is, we argue ``backwards'' from a given argument $B\in\alg$ down to its smallest subarguments (the singletons), and derive a contradiction.}.

We start with that $f^{-1}\pair{r}=\frac{\theta:\phi}{\phi}$ and $\theta\notin E$. As $r\in DR(B)$, and $r=TopRule(B')$ (Equation \ref{eq:attack}), there is a $B''\propsubarg B$ that concludes $\theta$. Either $B''$ is strict or it is not strict. Suppose it is strict, then $Prem(B'')\models\theta$, so by monotonicity $W\models\theta$ because $Prem(B'')\subseteq W$. This means $\theta\in E_0\subseteq E$ -- contradiction. Therefore, $B''$ cannot be strict.

Furthermore, either $B''\in Args(R)$ or not. If $B''\in Args(R)$, then $DR(B'')\subseteq R$, which means $\theta\in E_i\subseteq E$, where $i$ is the level such that all defaults corresponding to $DR(B'')$ have been applied (Equation \ref{eq:ext_ind}) -- contradiction. Therefore, $B''\notin Args(R)$.

This means $DR(B'')\not\subseteq R$, which means there is some rule, $s\in DR(B'')$, such that $f^{-1}(s)\notin NBD(E)$. Suppose $s=\frac{\theta':\phi'}{\phi'}$. There are two possibilities: either $\theta'\notin E$ or $\neg\phi'\in E$. If the latter, then we can construct an argument $A'$ concluding $\neg\phi'$ which defeats $B''$ as in the case when $\neg\phi\in E$. If the former, we can argue as in the previous paragraph to get a strictly smaller argument $B'''\subset_\text{arg} B''$ which concludes $\theta'$.

We cannot continue this process forever because arguments are well-founded. Eventually, we must stop at a strict subargument of $B''$, which gives a contradiction. Therefore, we cannot have the case $\theta\notin E$. Therefore, this second case is impossible, and the first case means that for every argument $B\notin Args(R)$ there is some $A\in Args(R)$ such that $A\hookrightarrow B$. This proves the first statement of the representation theorem.\\

\noindent\textit{2. $Conc(\ext)=E$:}\\

\noindent We show that $Conc(\ext)\subseteq\LForm$ is the prioritised default extension of our LPDT $\ang{D,\:W,\:\prec^+}$. We let $E$ be the prioritised default extension of $\ang{D,\:W,\:\prec^+}$ and show $E=Conc(\ext)$.\\

\noindent\textit{2.1 $Conc(\ext)\subseteq E$:}\\

\noindent We first show $Conc(\ext)\subseteq E$. Let $\theta\in Conc(\ext)$, which means there is some argument $A\in\ext$ where $Conc(A)=\theta$. Either $A$ is strict or it is not.

If $A$ is strict, then as $Prem(A)\subseteq W$, we must have $W\models\theta$ by monotonicity. Therefore, $W\in E_0\subseteq E$ by Equation \ref{eq:ext_base}, and hence $\theta\in E$.

If $A$ is not strict, then for some $k\in\nat^+$, $DR(A):=\set{d_1,\:\ldots,\:d_k}$. None of these defaults give rise to a conflict because $\ext$ is a stable extension. Take the smallest $i\in\nat$ such that sufficiently many corresponding defeasible rules are applied from $DR(A)$ to conclude $\theta$ in $E_{i+1}$ from $W$. Therefore, $\theta\in E_{i+1}\subseteq E$ and hence $\theta\in E$. Therefore, in either case, $Conc(\ext)\subseteq E$.\\

\noindent\textit{2.2 $E\subseteq Conc(\ext)$:}\\

\noindent We now show $E\subseteq Conc(\ext)$. Let $\theta\in E$ so $\pair{\exists i\in\nat}\:\theta\in E_i$ by Equation \ref{eq:ext_ind} (page \pageref{eq:ext_ind}). We have to show there is some argument $A\in\ext$ such that $Conc(A)=\theta$. Either $i=0$ or $i>0$.\\

\noindent\textit{2.2.1 The case of $i=0$:}\\

\noindent Suppose $i=0$, which means $\theta\in E_0\Leftrightarrow W\models\theta$. By compactness, we have some finite $W'\subseteq W$ such that $W'\models\theta$. We can build a strict argument $A$ with $Prem(A)=W'$ and conclusion $\theta$ as $\relsymb_s$ has all the appropriate rules of inference in FOL. Assume for contradiction $A\notin\ext$, then there exists some $B\in\ext$ defeating $A$, which is impossible because $A$ is strict. Therefore, $A\in\ext$ and $Conc(A)=\theta$, so $\theta\in Conc(\ext)$ by Equation \ref{eq:set_arg_attrb} (page \pageref{eq:set_arg_attrb}).\\

\noindent\textit{2.2.2 The case of $i>0$:}\\

\noindent Now suppose $i>0$. As $\theta\in E_i$, let $d_j$ for $0\leq j\leq i-1$ be the $\prec^+$-greatest active default in $E_j$. We can use the corresponding defeasible rules $r_j=f\pair{d_j}$ to build an argument $A$ such that $Prem(A)\subseteq W$, $Conc(A)=\theta$ and $DR(A)\subseteq\set{r_j}_{j=0}^{i-1}$. Now we need to show $A\in\ext$.

Assume for contradiction that $A\notin\ext$, then there is some $B\in\ext$ such that $B\defeat A$. So there is some defeasible rule $r$ in $A$ that is necessary to conclude $\theta$, such that $Conc(B)=\neg Cons(r)$. Either $B$ is strict or not.\\

\noindent\textit{2.2.2.1 If $B\in\ext$ is strict:}\\

\noindent Assume that $B$ is strict, then $Conc(B)\in E_0\subseteq E_i$, which must conflict with at least one of the rules in $DR(A)$. If this is so, then the corresponding defaults to these rules cannot be $\prec^+$-active in the appropriate $E_j$'s, and hence $A$ cannot be constructed - contradiction. Therefore, $B$ cannot be strict.\\

\noindent\textit{2.2.2.2 If $B\in\ext$ is not strict:}\\

\noindent Assume that $B$ is not strict, then $DR(B)\neq\es$ and, as $B\defeat A$, there is some $r\in DR(A)-DR(B)$ such that for all $s\in DR(B)-DR(A)$, $r<_{SP} s$ by Equation \ref{eq:SP_arg_pref} (page \pageref{eq:SP_arg_pref}). By Equation \ref{eq:SP_ord_def}, even if $r\in DR(A)-DR(B)$ can be added to the arguments as a defeasible rule, every single $s\in DR(B)-DR(A)$ is $<''$-more preferred than $r$. Therefore, the corresponding defaults in $DR(B)-DR(A)$ are $\prec^+$-greatest active in $E_j$ for $j<i$, the application of which would block $r$ from being applied. This contradicts the claim that it is possible to construct $A$ in order to conclude $\theta$. Therefore, $B$ cannot exist.

Therefore, $A\in\ext$, and given that $Conc(A)=\theta$, we have $\theta\in Conc(\ext)$ by Equation \ref{eq:set_arg_attrb}. As $\theta$ is arbitrary, we conclude $E\subseteq Conc(\ext)$ and hence $E=Conc(\ext)$. This proves that $Conc(\ext)$ is the prioritised default extension of $\ang{D,\:W,\:\prec^+}$.
\end{proof}

The representation theorem allows us to formally interpret the inferences of PDL as the conclusions of justified arguments, and the conclusions of the justified arguments are exactly those of the corresponding PDT. Therefore, by the representation theorem, PDL is sound and complete with respect to its argumentation semantics.

\subsection{Summary}

In this section, we have provided an instantiation of ASPIC$^+$ to PDL. We can construct an ASPIC$^+$ attack graph from a LPDT. The subtlety then is to find a suitable argument preference relation such that it gives a correspondence between the conclusions of the justified arguments, and the extensions of the PDT. We showed that none of the ASPIC$^+$ orders gives a correspondence, and even the intuitive disjoint elitist order does not give a correspondence either. We then devised the structure-preference order which mimics how defaults are added in PDL when constructing extensions. We then showed that ASPIC$^+$ defeat graphs that have been constructed by LPDTs have unique stable extensions. The representation theorem states that under the structure-preference order the inferences correspond exactly - this is a soundness and completeness result.

\section{On the Normative Rationality of this Instantiation}\label{sec:normative_rationality_current}

We have so far instantiated ASPIC$^+$ to PDL through an appropriate choice of the underlying logic, defeasible rules and preferences. We have proven that the inferences of this instantiation correspond exactly in Theorem \ref{thm:rep_thm} (page \pageref{thm:rep_thm}). In this section, we will discuss current work on establishing whether this instantiation is normatively rational\footnote{We say ``normatively rational'' to indicate that the type of rationality we are considering is prescriptive, not descriptive.}.

\subsection{Rational Instantiations of ASPIC\texorpdfstring{$^+$}{+}}

ASPIC$^+$ can in principle be instantiated into any concrete argumentation theory, but it is desirable for such instantiations to be normatively rational. For example, a normatively rational instantiation of ASPIC$^+$ would guarantee that the conclusions of the ultimately justified arguments are consistent. This idea of normative rationality for structured argumentation frameworks have been formalised by \cite{Caminada:07}. Let $\ext$ denote the set of justified arguments. The rationality postulates informally state \cite[Section 4.2]{sanjay:13}:
\begin{enumerate}
\item $\ext$ is subargument-closed.
\item $Conc\pair{\ext}$ is closed under strict rules.
\item $Conc\pair{\ext}$ is consistent.
\item Under $\relsymb_s$, the closure under strict rules of $Conc\pair{\ext}$ is consistent.
\end{enumerate}

\noindent ASPIC$^+$ provides sufficient conditions for its instantiations to be rational. They are:
\begin{enumerate}
\item The argumentation theory (i.e. the argumentation system and the knowledge base) is \textit{well-defined} \cite[page 369, Definition 12]{sanjay:13},
\item and the argument preference relation $\precsim$ is \textit{reasonable} \cite[page 372, Definition 18]{sanjay:13}.
\end{enumerate}

\subsection{Well-Definedness of this Instantiation}

A well-defined classical logic ASPIC$^+$ instantiation need only satisfy:
\begin{enumerate}
\item \textit{Closure under transposition:} If the rule $\pair{\theta_1,\:\ldots,\:\theta_n\to\phi}\in\relsymb_d$ for $n\in\nat$, then for all $1\leq i\leq n$,\[\pair{\theta_1,\:\ldots,\:\theta_{i-1},\:\neg\phi,\:\theta_{i+1},\:\ldots,\:\theta_n\to\neg\theta_i}\in\relsymb_s\:.\]This is satisfied because $\relsymb_s$ has all the rules of proof of FOL.
\item \textit{Axiom consistency:} This means $\relsymb_n$ is consistent, so we assert that $W$ is consistent\footnote{The argumentation semantics for PDL will still be valid for an inconsistent $W$, but normative rationality excludes this case by requiring $W$ to be consistent}. Given that we are considering LPDTs with consistent $W$, axiom consistency is satisfied.
\item \textit{Well-formed:} This is a property concerning asymmetric contrary functions, and is vacuously satisfied for instantiations with only a symmetric contrary function, like classical negation $\neg$.
\end{enumerate}

\noindent Therefore, this ASPIC$^+$ instantiation into PDL is well-defined.

\subsection{Reasonableness of the Argument Preference}

One further requirement for an ASPIC$^+$ instantiation to be normatively rationality is that the argument preference relation, $\precsim$, is \textit{reasonable}. For a more detailed discussion of why this is important, see \cite[Section 4]{sanjay:13}. In this subsection, we are concerned with whether the structure preference order, $\precsim_{SP}$, is reasonable.

\subsubsection{Strict Extensions}\label{sec:strict_extensions_of_sets_of_arguments}

(From \cite[page 370 Definition 17]{sanjay:13}) Let $S\subseteq_\text{fin}\alg$. We define a \textit{strict extension of $S$} to be any \textit{argument} $A\in\alg$ that satisfies
\begin{align}\label{eq:strict_extensions}
&DR(A)=DR(S),\:Prem_p(A)=Prem_p(S)\:,\nonumber\\
&SR(A)\supseteq SR(S)\text{ and }Prem_n(A)\supseteq Prem_n(S)\:.
\end{align}
The intuition is that given a \textit{finite} set of arguments $S$, we combine all of these arguments into a bigger argument $A$ only by adding strict rules and axiom premises, leaving the fallible information unchanged. The set $S$ must be finite because the requirement is that \textit{all} arguments in $S$ must be combined into a \textit{single} argument $A$. We may use the notation $S^+$ ($\in\alg$) instead of $A$ to emphasise that $S^+$ is a strict extension obtained by extending all arguments of the set $S$.

Note that given $S$, $S^+$ may not exist, nor must it be unique if it does exist. In the former case, there may not be suitable strict rules whose antecedents are in $Conc(S)$ such that $S^+$ is well-defined, in the sense that $S^+$ is a single argument. In the latter case, $\es^+$ exists when $\mathcal{K}_n\neq\es$ but is, by definition, any strict and firm argument (which includes singleton arguments). Clearly, every argument is its own strict extension.

We can define \textit{the set of strict extensions of (a finite set of arguments) $S$} to be
\begin{align}
StExt(S):=\set{A\in\alg\:\vline\:\text{Equation \ref{eq:strict_extensions} is true for }A\:.}
\end{align}
Clearly, $\set{A}\subseteq StExt\pair{\set{A}}$. %So when formal statements are made about all strict extensions of $S\subseteq_\text{fin}\alg$, we begin the statement with the bounded quantifier ``$\pair{\forall A\in StExt(S)}$''. Note then if $StExt(S)=\es$, i.e. when there are no well-defined strict extensions for the set $S$, then such a statement will be vacuously true.

\begin{eg}\label{eg:strict_ext_no_strict_rules}
Consider an instantiation where $\relsymb_s=\es$, i.e. there are \textit{no strict rules}. Given an arbitrary \textit{finite} set $S$ of arguments, what are the strict extensions in this case?
\begin{enumerate}
\item If $S=\es$, then $S^+$ is only defined when $\mathcal{K}_n\neq\es$, and $S^+$ is any of the (strict) singleton arguments. If $\mathcal{K}_n=\es$, then $\es^+=*$ and hence $StExt(S)=\es$.
\item If $|S|=1$, say $S=\set{A}$, then $StExt(S)=\set{A}$.
\item If $|S|\geq 2$, then $StExt(S)=\es$ because there are no strict rules to join multiple arguments (or additional axiom premises that may be introduced) in $S$ together into one argument.
\end{enumerate}
Therefore, when $\relsymb_s=\es$, the only case where the strict extension of a set of arguments $S$ is defined is when $S$ is singleton.
\end{eg}

\subsubsection{Reasonableness Defined}

\begin{rem2}
(From \cite[page 372, Definition 18]{sanjay:13}) An argument preference relation $\precsim$ on $\alg$ is \textit{reasonable} iff for all $A,B\in\alg$ and $\es\neq S\subseteq_\text{fin}\alg$,
\begin{enumerate}
\item (R1) If $A$ is strict and firm, and $B$ is neither strict nor firm, then $B\prec A$.
\item (R2) If $A$ is strict and firm, then $A\not\prec B$.
\item (R3) If $A\not\prec B$ then $\set{A}^+\not\prec B$. If $B\not\prec A$ then $B\not\prec\set{A}^+$ (for appropriate strict extensions).
\item (R4) It is \textit{not} the case that
\begin{align}\label{eq:reas_acyclic}
\pair{\forall A\in S}\pair{\exists B\in StExt\pair{S-\set{A}}}B\prec A\:.
\end{align}
\end{enumerate}
\end{rem2}

\noindent The intuition is as follows. (R1) and (R2) state that strict and firm arguments must be maximally preferred. (R3) states that strict extensions do not change the relative preference of arguments. (R4) is an acyclicity condition as illustrated in the following example:

\begin{eg}\label{eg:acyclic_illustrated}
Suppose we have a classical logic instantiation of ASPIC$^+$ where $\relsymb_s$ have all strict rules and $\neg$ is the only (symmetric) contrary function. Let $\relsymb_d=\set{r_1,r_2,r_3}$ such that $r_1:=(\top\Rightarrow a)$, $r_2:=(\top\Rightarrow b)$ and $r_3:=(\top\Rightarrow\neg(a\wedge b))$. Let $A:=[\top\Rightarrow a]$, $B:=[\top\Rightarrow b]$ and $C:=[\top\Rightarrow\neg(a\wedge b)]$. We also define $F:=[A,B\to(a\wedge b)]\in StExt\pair{\set{A,B}}$, $E:=[C,A\to\neg b]\in StExt\pair{\set{C,A}}$ and $D:=[B,C\to\neg a]\in StExt\pair{\set{B,C}}$. We illustrate these arguments in Figure \ref{figure:disj_eli_works}.

\begin{figure}[ht]
\begin{center}
\includegraphics[height=3.82cm,width=6cm]{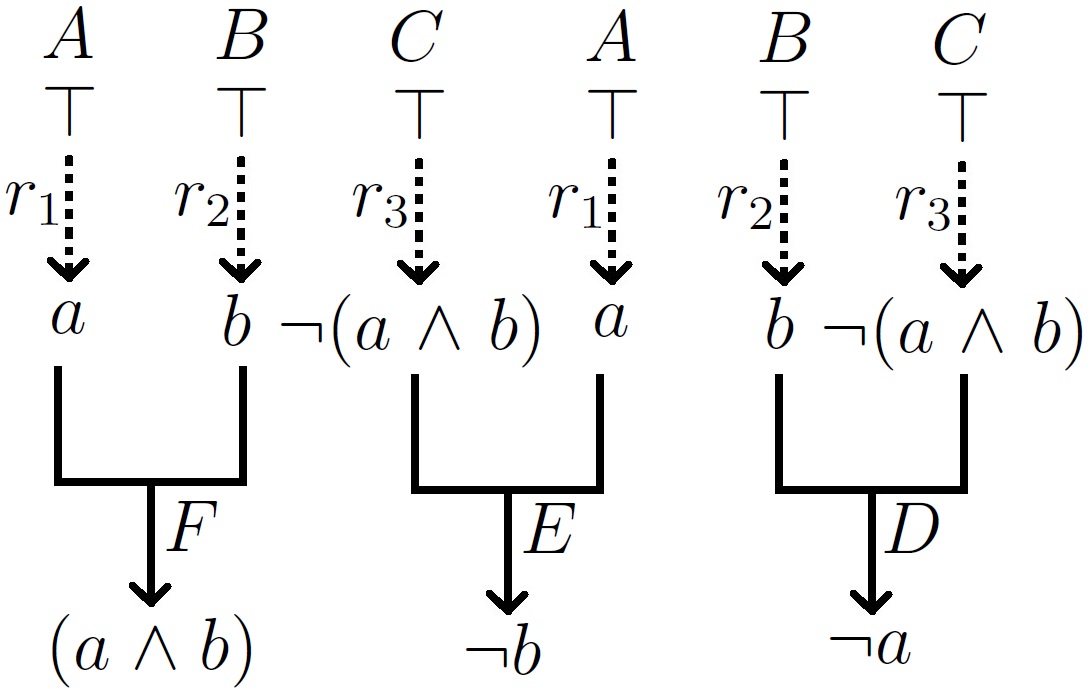}
\caption{The arguments in Example \ref{eg:acyclic_illustrated}.}
\label{figure:disj_eli_works}
\end{center}
\end{figure}

Clearly, $DR(A)=\set{r_1}$, $DR(B)=\set{r_2}$, $DR(C)=\set{r_3}$, $DR(D)=\set{r_2,r_3}$, $DR(E)=\set{r_3,r_1}$ and $DR(F)=\set{r_1,r_2}$.

Now assume (R4) is false, which means Equation \ref{eq:reas_acyclic} is \textit{true}. This means, for $S=\set{A,\:B,\:C}$,
\begin{align}\label{eq:acyclic_cond}
& D\prec A,\:E\prec B,\:F\prec C\nonumber\\
\Leftrightarrow&\set{r_2,r_3}\orddeneq\set{r_1},\:\set{r_3,r_1}\orddeneq\set{r_2},\:\set{r_1,r_2}\orddeneq\set{r_3}\nonumber\\
\Leftrightarrow&\sqbra{r_2<''r_1\text{ or }r_3<''r_1},\:\sqbra{r_3<''r_2\text{ or }r_1<''r_2},\:\sqbra{r_1<''r_3\text{ or }r_2<''r_3}\:,
\end{align}
but no \textit{total} order $<''$ on the set $\set{r_1,r_2,r_3}$ can satisfy any of the above eight conditions (Equation \ref{eq:acyclic_cond}), as a cycle will always be created. Therefore, $\precsim$ based on the disjoint elitist order satisfies (R4) in this example.

%Of course, there may be other sets $S\subseteq_\text{fin}\alg$
\end{eg}

\subsubsection{Is the Structure Preference Order Reasonable?}

\begin{lem}\label{lem:disj_eli_almost_reas}
The structure preference order over arguments, $\precsim_{SP}$ (Equation \ref{eq:SP_arg_pref}, page \pageref{eq:SP_arg_pref}), satisfies (R1) to (R3) in the definition of reasonableness\footnote{In fact, this result holds for the disjoint elitist order in general, i.e. with any underlying strict total order $<''$. Recall that $\precsim_{SP}$ is just the disjoint elitist order with $<_{SP}$ (Equation \ref{eq:SP_order}, page \pageref{eq:SP_order}) as the underlying strict total order.}.
\end{lem}
\begin{proof}
We have:
\begin{enumerate}
\item (R1) This follows because $\es$ is the greatest element under the disjoint elitist order.
\item (R2) This follows for the same reason as (R1).
\item (R3) This follows because strict extensions do not change the set of defeasible rules.
\end{enumerate}
This shows the result.
\end{proof}

So to show that $\precsim_{SP}$ is reasonable, we need to show (R4) is true. This is work in progress and we now provide some special cases.

\begin{cor}\label{cor:R2_implies_R4_for_S_singleton}
If $|S|\leq 2$, then $\precsim_{SP}$ is reasonable\footnote{This result holds for the disjoint elitist order more generally.}.
\end{cor}
\begin{proof}
Assume for contradiction that $\precsim_{SP}$ is not reasonable, i.e. Equation \ref{eq:reas_acyclic} is \textit{true}.

If $|S|=1$, then $S=\set{A}$ (say) so there is a $B\in StExt\pair{S-\set{A}}=StExt(\es)$ such that $B\prec_{SP}A$, so there is a strict (and firm) argument $B$ that is strictly less preferred than $A$, which contradicts (R2).

If $|S|=2$, then $S=\set{A,B}$ (say), so there is a $C\in StExt\pair{\set{B}}$ such that $C\prec_{SP} A$, and there is a $D\in StExt\pair{\set{A}}$ such that $D\prec_{SP} B$. But by the contrapositive of (R3) this means $A\prec_{SP} B\prec_{SP} A$, which contradicts irreflexivity.

Therefore, for the case of $|S|\leq 2$, $\precsim_{SP}$ is reasonable.
\end{proof}

\begin{lem}
If $\relsymb_s=\es$, then $\precsim_{SP}$ is reasonable\footnote{This result holds for the disjoint elitist order more generally.}.
\end{lem}
\begin{proof}
From Lemma \ref{lem:disj_eli_almost_reas}, it is sufficient to show $\precsim_{SP}$ satisfies (R4). Let $\es\neq S\subseteq_\text{fin}\alg$ be arbitrary. We need to show it is \textit{not} the case that, for any $A\in S$, there is some $B\in StExt\pair{S-\set{A}}$ such that $B\prec_{SP} A$. Assume for contradiction that it is true, and let $A\in S$ be arbitrary. We know from Corollary \ref{cor:R2_implies_R4_for_S_singleton} (page \pageref{cor:R2_implies_R4_for_S_singleton}) we need to show this for $|S|>2$. But from Example \ref{eg:strict_ext_no_strict_rules} (page \pageref{eg:strict_ext_no_strict_rules}), in the case of no strict rules, $StExt(S-\set{A})=\es$ for $|S|>2$, which means Equation \ref{eq:reas_acyclic} is false, therefore (R4) holds.
\end{proof}

\begin{lem}
If $\es\neq S\subseteq_\text{fin}\alg$ is such that for all $A,\:B\in S$,
\begin{align}
A\neq B\implies DR(A)\cap DR(B)=\es\:,
\end{align}
then $\precsim_{SP}$ is reasonable\footnote{This result holds for the disjoint elitist order more generally.}.
\end{lem}
\begin{proof}
Clearly $DR(S)\subseteq_\text{fin}\relsymb_d$, because $DR(A)\subseteq_\text{fin}\relsymb_d$ for each $A\in\alg$ and $DR(S)$ is a union of finite sets. Given the strict toset $\ang{\relsymb_d,\:<''}$, we also have a \textit{finite} strict toset $\ang{DR(S),\:<''}$. This has a $<''$-least element $r_0\in DR(S)$. By definition, $\pair{\exists A\in S}r_0\in DR(A)$ (Equation \ref{eq:set_arg_attrb}, page \pageref{eq:set_arg_attrb}). Call the witness to $\exists$ $A_0$, say.

Note that for all \textit{other} arguments in $S$ distinct from $A_0$, $r_0$ would not be amongst their defeasible rules. Now assume for contradiction (R4) is false, so Equation \ref{eq:reas_acyclic} is \textit{true}. We instantiate $\forall$ to $A_0$ and get
\begin{align}
\pair{\exists B\in StExt\pair{S-\set{A_0}}}B\prec_{SP} A_0\:.
\end{align}
Let $B_0$ be the witness to $\exists$. From the definition of the disjoint elitist order, $B_0\prec_{SP} A_0$ means that
\begin{align*}
&\pair{\exists x\in DR(B_0)-DR(A_0)}\pair{\forall y\in DR(A_0)-DR(B_0)}x<''y\\
\implies&\pair{\exists x\in DR(B_0)}x<'' r_0\:,
\end{align*}
because $r_0\in DR(A_0)$ and $r_0\notin DR(B_0)$. However, as $r_0$ is the $<''$-least element of $DR(S)$, there is no element in $DR(B_0)$ that is smaller than $r_0$. Therefore, Equation \ref{eq:reas_acyclic} is false and hence (R4) is true.
\end{proof}

Notice that this last result is consistent with Example \ref{eg:acyclic_illustrated} (page \pageref{eg:acyclic_illustrated}). It is still unknown whether $\precsim_{SP}$ is reasonable for general sets $S$ where $\es\neq S\subseteq_\text{fin}\alg$ although we conjecture that it should be given the consistency properties of PDL and that we have shown the representation theorem. This is work in progress.

\subsection{Summary}

In this section, we have reviewed the sufficient conditions that an ASPIC$^+$ instantiation needs to satisfy in order to be normatively rational, which formally means that the rationality postulates of \cite{Caminada:07} are true. The sufficient conditions are that the instantiation is well-defined, and the argument preference relation is reasonable. It is easy to show that the PDL instantiation is well-defined. We are currently working on showing how the structure preference order, $\precsim_{SP}$, is reasonable. We conjecture that it is reasonable, due to the representation theorem.

\section{Discussion and Conclusion}\label{sec:discussion_conclusions}

\noindent In this note we have endowed PDL \cite{Brewka:94} with argumentation semantics using ASPIC$^+$ \cite{sanjay:13}. We did this by instantiating ASPIC$^+$ to PDL (Section \ref{sec:ASPIC+_to_PDL}, page \pageref{sec:ASPIC+_to_PDL}), devising an ASPIC$^+$ preference order that imitates the procedural construction of extensions in PDL (Section \ref{sec:SP_order}, page \pageref{sec:SP_order}), and proving the conclusions of the justified arguments correspond exactly to the inferences in PDL (Theorem \ref{thm:rep_thm}, page \pageref{thm:rep_thm}) under this preference. As explained in Section \ref{sec:intro}, endowing PDL with argumentation semantics allows us to perform inferences in PDL dialectically, in the sense that inference in PDL can formally be viewed as a process of exchanging of arguments and counterarguments \cite{Sanjay:09}, until the ultimately justified arguments are found, the conclusions of which are exactly what PDL would conclude. This renders the process of inference in PDL more intuitive, and clarifies the reasons for accepting or rejecting a conclusion.

It is easy to see how Theorem \ref{thm:rep_thm} generalises the argumentation semantics of preferred subtheories \cite[page 381 Theorem 34]{sanjay:13}. Informally, a default theory is isomorphic to a PDT $\ang{D,\:\es,\:\prec}$, where $D$ consists of supernormal defaults and $\prec$ is consistent with how the sets of the default theory are ranked. Given a linearisation $\prec^+\:\supseteq\:\prec$, the corresponding preferred subtheory $\Sigma$ is related to the prioritised default extension by $E=Th\pair{\Sigma}$. The set of arguments with premises from $\Sigma$ is graph-isomorphic to $Args\pair{f\pair{NBD(E)}}$, both of which form a stable extension. Similarly, given the stable extension $\ext$, $Conc\pair{\ext}$ is the prioritised default extension by Theorem \ref{thm:rep_thm}, and the conclusions of the defeasible rules that feature in $\ext$ make up the corresponding preferred subtheory. In future work we will articulate this idea more formally.

There are several issues with the approach we have taken. Firstly, it seems that we have lost generality because we have restricted attention to LPDTs (Section \ref{sec:inst}, page \pageref{sec:inst}). We claim that this does not lose generality because extensions in PDL always presuppose a linearisation $\prec^+$ of $\prec$ \cite{Brewka:94}, and we have shown that for \emph{any} such linearisation the correspondence between PDL and its argumentation semantics is exact. 

Secondly, we have not yet shown that the argument preference relation used, $\precsim_{SP}$, is reasonable, so we have not guaranteed normative rationality from the point of view of ASPIC$^+$. This is work in progress.

The importance of proving that $\precsim_{SP}$ is reasonable is that one can use ASPIC$^+$ to generalise PDL by abstracting the concepts developed in this note to other situations, not necessarily motivated by PDL. For example, if $\precsim_{SP}$ is reasonable, then it can be used in a wider range of contexts. Further, ASPIC$^+$ can identify argumentation-based inferences assuming only a partial ordering, unlike in PDL. How can multiple partial orderings be related to multiple stable extensions of PDL without explicitly linearising? Also, do the other Dung semantics\footnote{When the extension is unique, the distinction between the different Dung semantic types is lost.} become relevant? All of this suggests that our argumentative characterisation can be used to generalise PDL, yet if we lift the requirement to linearise, we can no longer guarantee normative rationality, because one can show that the disjoint elitist order is not transitive when the underlying set $\ang{\relsymb_d,\:<''}$ is a poset instead of a toset\footnote{See Appendix \ref{app:disj_eli_props}, Lemma \ref{lem:disj_eli_not_transitive}, page \pageref{lem:disj_eli_not_transitive}.}. Future work will consider how to generalise the requirement that the defeasible rules are totally ordered, how in this case one can obtain all PDL extensions via argumentation, and the significance of other types of Dung semantics.

Lastly, the argumentative characterisation of PDL provides for distributed reasoning in the course of deliberation and persuasion dialogues. For example, BOID agents with PDL representations of mental attitudes can now exchange arguments and counterarguments when deliberating about which goals to select, and consequently which actions to pursue. Future work can investigate the precise advantages the argumentation semantics in PDL have over more traditional approaches in such situations.

\bibliographystyle{abbrv}
\bibliography{APY_PhD_Library}

\appendix

\section{Properties of the Disjoint Elitist Order}\label{app:disj_eli_props}

In this section we prove several statements made in Section \ref{sec:disj_eli} (page \pageref{sec:disj_eli}).

\begin{thm}
If $\ang{P,\:<}$ is a strict toset, then $\ang{\powfin(P),\:\orddeneq}$ is also a strict toset.
\end{thm}
\begin{proof}
We prove that $\orddeneq$ is irreflexive, transitive and total over $\powfin(P)$.

Assume for contradiction $\Gamma\orddeneq\Gamma$, which is equivalent to, by Equation \ref{eq:disj_eli} (page \pageref{eq:disj_eli}), $\pair{\exists x\in\es}\pair{\forall y\in\es}\:x<y$, which is impossible because exists precedes for all. Therefore, $\orddeneq$ is irreflexive.

To show transitivity, let $n_1,\:\cdots,\:n_7\in\nat$, such that
\begin{align}
&\set{a_1,\:\cdots,\:a_{n_1}}\cup\set{b_1,\:\cdots,\:b_{n_2}}\cup\set{c_1,\:\cdots,\:c_{n_3}}\cup\set{d_1,\:\cdots,\:d_{n_4}}\nonumber\\
\cup&\set{e_1,\:\cdots,\:e_{n_5}}\cup\set{f_1,\:\cdots,\:f_{n_6}}\cup\set{g_1,\:\cdots,\:g_{n_7}}\subseteq P\:.
\end{align}
Each element of the sets are distinct. If $n_i=0$ then the corresponding set is empty. It is sufficient to consider finite subsets due to $\powfin(P)$. Let $\Gamma,\:\Gamma',\:\Gamma''$ be such that
\begin{align*}
\Gamma&=\set{a_1,\:\cdots,\:a_{n_1}}\cup\set{d_1,\:\cdots,\:d_{n_4}}\cup\set{f_1,\:\cdots,\:f_{n_6}}\cup\set{g_1,\:\cdots,\:g_{n_7}}\:,\\
\Gamma'&=\set{b_1,\:\cdots,\:b_{n_2}}\cup\set{d_1,\:\cdots,\:d_{n_4}}\cup\set{e_1,\:\cdots,\:e_{n_5}}\cup\set{g_1,\:\cdots,\:g_{n_7}}\text{ and}\\
\Gamma''&=\set{c_1,\:\cdots,\:c_{n_3}}\cup\set{e_1,\:\cdots,\:e_{n_5}}\cup\set{f_1,\:\cdots,\:f_{n_6}}\cup\set{g_1,\:\cdots,\:g_{n_7}}\:.
\end{align*}

\noindent We can picture this situation with the the Venn diagram in Figure \ref{figure:Venn}.

\begin{figure}[ht]
\begin{center}
\includegraphics[height=5.55cm,width=8cm]{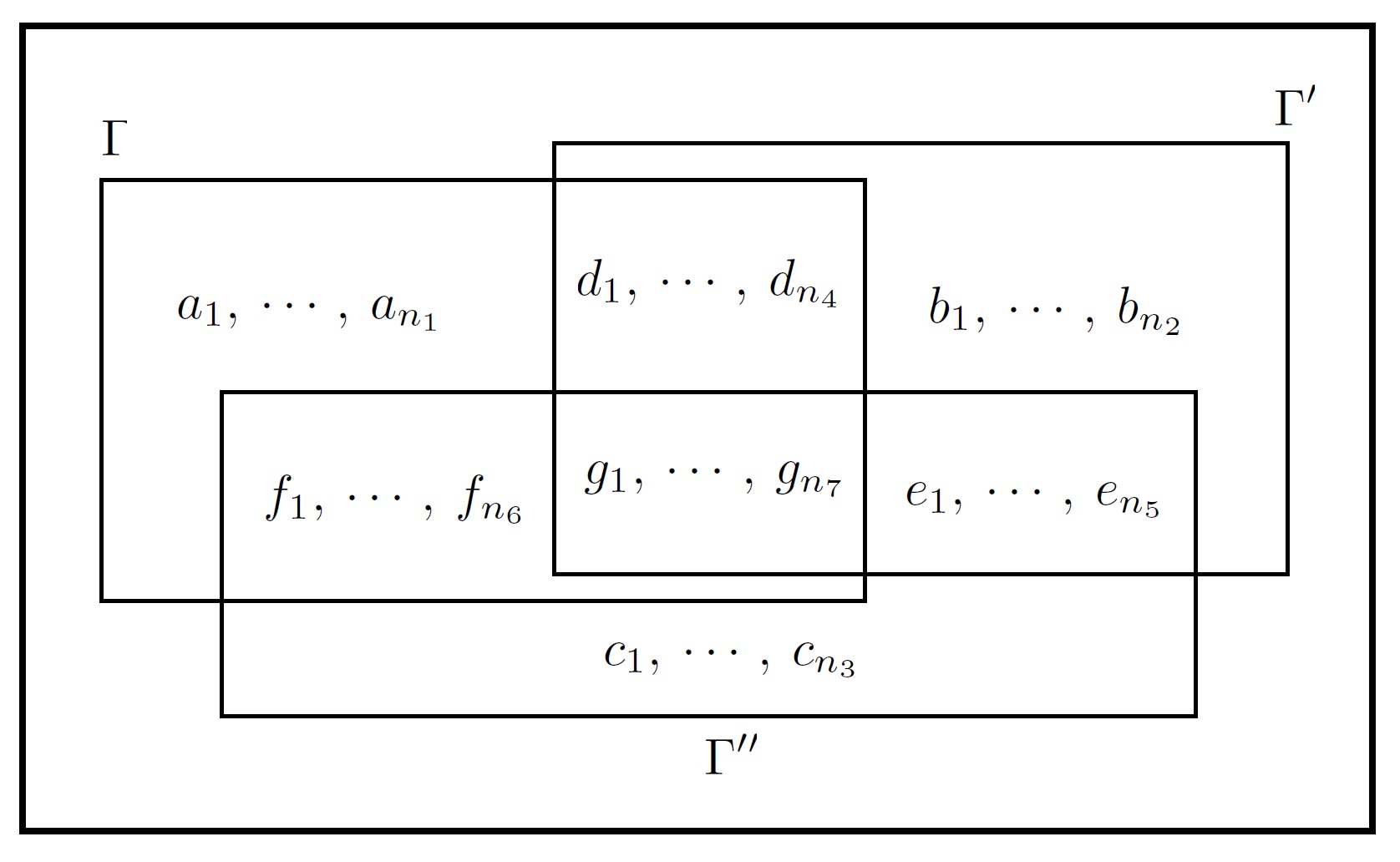}
\caption{The solid outer rectangle represents the set $P$, which may be an infinite set. The three finite sets $\Gamma,\:\Gamma',\:\Gamma''$ are the three rectangles within. Each overlapping region has exactly the elements indicated (and nothing more).}
\label{figure:Venn}
\end{center}
\end{figure}

Now suppose $<$ permits $\Gamma\orddeneq\Gamma'\orddeneq\Gamma''$, we write this out in terms of the elements (Equations \ref{eq:first_ass_trans} and \ref{eq:second_ass_trans}).
\begin{align}\label{eq:first_ass_trans}
\Gamma\orddeneq\Gamma'\Leftrightarrow&\pair{\exists x\in\Gamma-\Gamma'}\pair{\forall y\in\Gamma'-\Gamma}\:x<y\nonumber\\
\Leftrightarrow&\pair{\exists x\in\set{a_1,\:\cdots,\:a_{n_1}}\cup\set{f_1,\:\cdots,\:f_{n_6}}}\nonumber\\
&\pair{\forall y\in\set{b_1,\:\cdots,\:b_{n_2}}\cup\set{e_1,\:\cdots,\:e_{n_5}}}\:x<y\nonumber\\
\Leftrightarrow&\pair{\exists x\in\set{a_1,\:\cdots,\:a_{n_1}}\cup\set{f_1,\:\cdots,\:f_{n_6}}}\nonumber\\
&\sqbra{\pair{\bigwedge_{i=1}^{n_2}x<b_i}\wedge\pair{\bigwedge_{j=1}^{n_5}x<e_j}}\nonumber\\
\Leftrightarrow&\bigvee_{k=1}^{n_1}\sqbra{\pair{\bigwedge_{i=1}^{n_2}a_k<b_i}\wedge\pair{\bigwedge_{j=1}^{n_5}a_k<e_j}}\nonumber\\
\vee&\bigvee_{l=1}^{n_6}\sqbra{\pair{\bigwedge_{i=1}^{n_2}f_l<b_i}\wedge\pair{\bigwedge_{j=1}^{n_5}f_l<e_j}}\:.
\end{align}
\noindent Note that there are $(n_1+n_6)$ disjuncts in Equation \ref{eq:first_ass_trans}.

\begin{align}\label{eq:second_ass_trans}
\Gamma'\orddeneq\Gamma''\Leftrightarrow&\pair{\exists x\in\Gamma'-\Gamma''}\pair{\forall y\in\Gamma''-\Gamma'}\:x<y\nonumber\\
\Leftrightarrow&\pair{\exists x\in\set{b_1,\:\cdots,\:b_{n_2}}\cup\set{d_1,\:\cdots,\:d_{n_4}}}\nonumber\\
&\pair{\forall y\in\set{c_1,\:\cdots,\:c_{n_3}}\cup\set{f_1,\:\cdots,\:f_{n_6}}}\:x<y\nonumber\\
\Leftrightarrow&\pair{\exists x\in\set{b_1,\:\cdots,\:b_{n_2}}\cup\set{d_1,\:\cdots,\:d_{n_4}}}\nonumber\\
&\sqbra{\pair{\bigwedge_{i=1}^{n_3}x<c_i}\wedge\pair{\bigwedge_{j=1}^{n_6}x<f_j}}\nonumber\\
\Leftrightarrow&\bigvee_{k=1}^{n_2}\sqbra{\pair{\bigwedge_{i=1}^{n_3}b_k<c_i}\wedge\pair{\bigwedge_{j=1}^{n_6}b_k<f_j}}\nonumber\\
\vee&\bigvee_{l=1}^{n_4}\sqbra{\pair{\bigwedge_{i=1}^{n_3}d_l<c_i}\wedge\pair{\bigwedge_{j=1}^{n_6}d_l<f_j}}\:.
\end{align}
\noindent Note that there are $(n_2+n_4)$ disjuncts in Equation \ref{eq:first_ass_trans}.

We need to show that $\Gamma\orddeneq\Gamma''$, i.e.
\begin{align}\label{eq:target}
\Gamma\orddeneq\Gamma''\Leftrightarrow&\pair{\exists x\in\Gamma-\Gamma''}\pair{\forall y\in\Gamma''-\Gamma}\:x<y\nonumber\\
\Leftrightarrow&\pair{\exists x\in\set{a_1,\:\cdots,\:a_{n_1}}\cup\set{d_1,\:\cdots,\:d_{n_4}}}\nonumber\\
&\pair{\forall y\in\set{c_1,\:\cdots,\:c_{n_3}}\cup\set{e_1,\:\cdots,\:e_{n_5}}}\:x<y\nonumber\\
\Leftrightarrow&\pair{\exists x\in\set{a_1,\:\cdots,\:a_{n_1}}\cup\set{d_1,\:\cdots,\:d_{n_4}}}\nonumber\\
&\sqbra{\pair{\bigwedge_{i=1}^{n_3}x<c_i}\wedge\pair{\bigwedge_{j=1}^{n_5}x<e_j}}\nonumber\\
\Leftrightarrow&\bigvee_{k=1}^{n_1}\sqbra{\pair{\bigwedge_{i=1}^{n_3}a_k<c_i}\wedge\pair{\bigwedge_{j=1}^{n_5}a_k<e_j}}\nonumber\\
\vee&\bigvee_{l=1}^{n_4}\sqbra{\pair{\bigwedge_{i=1}^{n_3}d_l<c_i}\wedge\pair{\bigwedge_{j=1}^{n_5}d_l<e_j}}\:.
\end{align}
To prove Equation \ref{eq:target}, we need to show one of the disjuncts, i.e. for at least one of $1\leq k\leq n_1$ or $1\leq l\leq n_4$, we show either
\begin{align}\label{eq:answer_proof_trans}
\sqbra{\pair{\bigwedge_{i=1}^{n_3}a_k<c_i}\wedge\pair{\bigwedge_{j=1}^{n_5}a_k<e_j}}\text{ or }\sqbra{\pair{\bigwedge_{i=1}^{n_3}d_l<c_i}\wedge\pair{\bigwedge_{j=1}^{n_5}d_l<e_j}}
\end{align}
by establishing all of the conjuncts. Given $\Gamma\orddeneq\Gamma'\orddeneq\Gamma''$, we take the conjunction of Equations \ref{eq:first_ass_trans} and \ref{eq:second_ass_trans}, making $(n_1+n_6)(n_2+n_4)$ disjuncts, which is equivalent to the following expression:
\begin{align*}
&\set{\bigvee_{k=1}^{n_1}\sqbra{\pair{\bigwedge_{i=1}^{n_2}a_k<b_i}\wedge\pair{\bigwedge_{j=1}^{n_5}a_k<e_j}}\vee\bigvee_{l=1}^{n_6}\sqbra{\pair{\bigwedge_{i=1}^{n_2}f_l<b_i}\wedge\pair{\bigwedge_{j=1}^{n_5}f_l<e_j}}}\\
\wedge&\set{\bigvee_{k=1}^{n_2}\sqbra{\pair{\bigwedge_{i=1}^{n_3}b_k<c_i}\wedge\pair{\bigwedge_{j=1}^{n_6}b_k<f_j}}\vee\bigvee_{l=1}^{n_4}\sqbra{\pair{\bigwedge_{i=1}^{n_3}d_l<c_i}\wedge\pair{\bigwedge_{j=1}^{n_6}d_l<f_j}}}.
\end{align*}
As $\wedge$ and $\vee$ bi-distribute, we have four cases:
\begin{enumerate}
\item For some $1\leq k\leq n_1$ and $1\leq k'\leq n_2$, we have
\begin{align}
\pair{\bigwedge_{i=1}^{n_2}a_k<b_i}\wedge\pair{\bigwedge_{j=1}^{n_5}a_k<e_j}\wedge\pair{\bigwedge_{i'=1}^{n_3}b_{k'}<c_{i'}}\wedge\pair{\bigwedge_{j'=1}^{n_6}b_{k'}<f_{j'}}\:.
\end{align}
This means for some $1\leq k\leq n_1$, we have
\begin{align}\label{eq:1.1}
\pair{\bigwedge_{j=1}^{n_5}a_k<e_j}
\end{align}
and from\[\pair{\bigwedge_{i=1}^{n_2}a_k<b_i}\wedge\pair{\bigwedge_{i'=1}^{n_3}b_{k'}<c_{i'}}\:,\]that $1\leq k'\leq n_2$, and transitivity of $<$, we infer
\begin{align}\label{eq:1.2}
\pair{\bigwedge_{i=1}^{n_3}a_k<c_i}\:.
\end{align}
Equations \ref{eq:1.1} and \ref{eq:1.2} imply $\Gamma\orddeneq\Gamma''$.
\item For some $1\leq k\leq n_1$ and $1\leq l\leq n_4$, we have
\begin{align}
\pair{\bigwedge_{i=1}^{n_2}a_k<b_i}\wedge\pair{\bigwedge_{j=1}^{n_5}a_k<e_j}\wedge\pair{\bigwedge_{i'=1}^{n_3}d_l<c_{i'}}\wedge\pair{\bigwedge_{j'=1}^{n_6}d_l<f_{j'}}
\end{align}
This is the most subtle case of the four, and uses the fact that $<$ is total. The second the third bracketed conjuncts are necessary but not sufficient to lead to $\Gamma\orddeneq\Gamma''$. Let $k_0$ be the witness to $1\leq k\leq n_1$ and $l_0$ be the witness to $1\leq l_0\leq n_4$. As $<$ is total, either $a_{k_0}<d_{l_0}$ or $d_{l_0}<a_{k_0}$ (remember all elements are distinct).
\begin{itemize}
\item If $a_{k_0}<d_{l_0}$ then $a_{k_0}<c_i$ for all $1\leq i\leq n_3$. Therefore,\[\pair{\bigwedge_{i=1}^{n_3}a_{k_0}<c_i}\:.\]
\item If $d_{l_0}<a_{k_0}$ then $d_{l_0}<e_j$ for all $1\leq j\leq n_5$. Therefore,\[\pair{\bigwedge_{j=1}^{n_5}d_{l_0}<e_j}\:.\] 
\end{itemize}
In either case, $\Gamma\orddeneq\Gamma''$.
\item For some $1\leq l\leq n_6$ and $1\leq k\leq n_2$, we have
\begin{align}
\pair{\bigwedge_{i=1}^{n_2}f_l<b_i}\wedge\pair{\bigwedge_{j=1}^{n_5}f_l<e_j}\wedge\pair{\bigwedge_{i'=1}^{n_3}b_k<c_{i'}}\wedge\pair{\bigwedge_{j'=1}^{n_6}b_k<f_{j'}}
\end{align}
The irreflexivity of $<$ and the first and last bracketed conjuncts gives a contradiction when you run over all indices, so this case gives a contradiction.
\item For some $1\leq l\leq n_6$ and $1\leq l'\leq n_4$, we have
\begin{align}
\pair{\bigwedge_{i=1}^{n_2}f_l<b_i}\wedge\pair{\bigwedge_{j=1}^{n_5}f_l<e_j}\wedge\pair{\bigwedge_{i'=1}^{n_3}d_{l'}<c_{i'}}\wedge\pair{\bigwedge_{j'=1}^{n_6}d_{l'}<f_{j'}}
\end{align}
This case is similar to the first case - we use transitivity to combine the second and last bracketed conjuncts. This infers the second conjunct of Equation \ref{eq:answer_proof_trans}, which means $\Gamma\orddeneq\Gamma''$.
\end{enumerate}
Therefore, in all cases, $\Gamma\orddeneq\Gamma''$. This shows $\orddeneq$ is transitive when the underlying order $<$ is total.

Now let $\Gamma,\:\Gamma'\in\powfin(P)$ be arbitrary. To show trichotomy, we start by assuming $\Gamma\neq\Gamma'$ and show either $\Gamma\orddeneq\Gamma'$ or $\Gamma'\orddeneq\Gamma$. From Equation \ref{eq:disj_eli} (page \pageref{eq:disj_eli}), we consider the symmetric difference $\Gamma\ominus\Gamma'$. The set $\ang{\Gamma\ominus\Gamma',\:<}\:\subseteq\ang{P,\:<}$ is also a finite strict toset. This means there must exist a $<$-least element $x_0\in\Gamma\ominus\Gamma'$, say. There are two possibilities:
\begin{itemize}
\item If $x_0\in\Gamma-\Gamma'$, then $\Gamma\orddeneq\Gamma'$.
\item If $x_0\in\Gamma'-\Gamma$, then $\Gamma'\orddeneq\Gamma'$.
\end{itemize}
This establishes trichotomy - so $\ang{\powfin(P),\:\orddeneq}$ is a strict toset.
\end{proof}

\begin{lem}\label{lem:disj_eli_not_transitive}
If $\ang{P,\:<}$ is a strict poset, then $\orddeneq$ is not necessarily transitive over $\powfin(P)$.
\end{lem}
\begin{proof}
We provide the following counterexample: let $P=\set{a_0,\:a_1,\:a_2,\:a_3}$ such that $a_0<a_2$ and $a_1<a_3$ and nothing else. This means $a_0||a_3$ and $a_1||a_2$. This is a well-defined strict poset. Now let $\Gamma:=\set{a_0,\:a_1}$, $\Gamma':=\set{a_1,\:a_2}$ and $\Gamma'':=\set{a_2,\:a_3}$. We can illustrate this in Figure \ref{figure:disj_eli_not_trans}.

\begin{figure}[ht]
\begin{center}
\includegraphics[height=3cm,width=3.15cm]{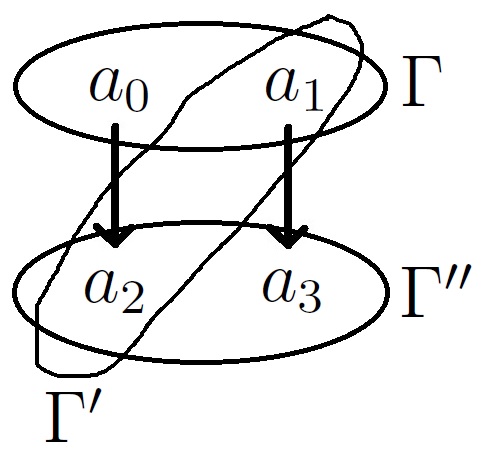}
\caption{The Hasse diagram for Lemma \ref{lem:disj_eli_not_transitive}. Note that in this document, our Hasse diagrams point the opposite direction, i.e. smaller elements are towards the top, and larger elements are towards the bottom.}
\label{figure:disj_eli_not_trans}
\end{center}
\end{figure}

Clearly, $\Gamma\orddeneq\Gamma'$ because $a_0<a_2$, and $\Gamma'\orddeneq\Gamma''$ because $a_1<a_3$. Now,
\begin{align*}
\Gamma\not\!\orddeneq\Gamma''\Leftrightarrow&\pair{\forall x\in\Gamma-\Gamma''}\pair{\exists y\in\Gamma''-\Gamma}x\not< y\\
\Leftrightarrow&\pair{\forall x\in\set{a_0,\:a_1}}\pair{\exists y\in\set{a_2,\:a_3}}\:x\not< y\\
\Leftrightarrow&\pair{\forall x\in\set{a_0,\:a_1}}\sqbra{x\not< a_2\text{ or }x\not< a_3}\\
\Leftrightarrow&\sqbra{a_0\not< a_2\text{ or }a_0\not< a_3}\text{ and }\sqbra{a_1\not< a_2\text{ or }a_1\not< a_3}\:.
\end{align*}
This is true, because $a_0\not< a_3$ and $a_1\not< a_2$ are both true. So this is a situation where $\Gamma\orddeneq\Gamma'\orddeneq\Gamma''$ and $\Gamma\not\!\orddeneq\Gamma''$. Therefore, $\orddeneq$ is not transitive.
\end{proof}

\begin{cor}
If $\ang{P,\:<}$ is a strict toset, then $\es$ is the $\orddeneq$-greatest element in $\ang{\powfin(P),\:\orddeneq}$.
\end{cor}
\begin{proof}
Assume for contradiction that there is some $\Gamma\in\powfin(P)$, $\es\orddeneq\Gamma$, which by Equation \ref{eq:disj_eli} (page \pageref{eq:disj_eli}) is equivalent to $\pair{\exists x\in\es}\pair{\forall y\in\Gamma}x<y$, which is false due to ``$\pair{\exists x\in\es}$''.
\end{proof}

\begin{cor}
If $\ang{P,\:<}$ is a strict \textit{finite} toset, then $P$ is the $\orddeneq$-least element in $\ang{\powfin(P),\:\orddeneq}$.
\end{cor}
\begin{proof}
Assume for contradiction that there is some $\Gamma\in\powfin(P)$, $\Gamma\orddeneq P$. As $\Gamma\subseteq P$, we must have $\Gamma-P=\es$. By Equation \ref{eq:disj_eli}, $\Gamma\orddeneq P$ is equivalent to $\pair{\exists x\in\es}\pair{\forall y\in\Gamma}$, which is false due to ``$\pair{\exists x\in\es}$''.
\end{proof}

\section{Characterising Non-Blocked Defaults}\label{app:NBD_char}

In this section we prove Lemma \ref{lem:NBD_characterisation} (page \pageref{lem:NBD_characterisation}). We restate the lemma below for convenience.

\begin{lem}
Let $\ang{D,\:W,\:\prec^+}$ be a LPDT. If $\prec^+$ generates $E$, then we have that
\begin{align}\label{eq:lem_NBD_app}
NBD(\prec^+):=\set{\frac{\theta:\phi}{\phi}\in D\:\vline\:\theta\in E,\:\neg\phi\notin E}\:.
\end{align}
\end{lem}
\begin{proof}
It is sufficient to show Equation \ref{eq:NBD} (with Equations \ref{eq:GD} and \ref{eq:SAD}) is the same as the right hand side of Equation \ref{eq:lem_NBD}. Let $\prec^+$ generate the extension $E$ and, for convenience, we suppress the argument ``$\prec^+$'' in the sets for this proof.

($\Rightarrow$) Case 1: Assume $d\in SAD$, then $Ante(d)\subseteq S,\:\neg Conc(d)\notin S$ and $Conc(d)\in S$. This implies $Ante(d)\subseteq S$ and $\neg Conc(d)\notin S$. Therefore,
\begin{align}\label{eq:NBD_proof_inter}
d\in\set{d'\in D\:\vline\:Ante(d')\subseteq E,\:\neg Conc(d')\notin E}\:,
\end{align}
\noindent and hence
\begin{align}\label{eq:SAD_subset_alternative_form}
SAD\subseteq\set{d\in D\:\vline\:Ante(d)\subseteq E,\:\neg Conc(d)\notin E}\:.
\end{align}

Case 2: Now assume $d\in GD$, which means
\begin{align}\label{eq:intermediate_d_is_active}
\Leftrightarrow&\pair{\exists i\in\nat}d\in GD_i\nonumber\\
\Leftrightarrow&\pair{\exists i\in\nat}\sqbra{Ante(d)\subseteq E_i,\: Conc(d)\notin E_i,\:\neg Conc(d)\notin E_i}\nonumber\\
\Leftrightarrow& Ante(d)\subseteq E_{j_0},\:Conc(d)\notin E_{j_0},\:\neg Conc(d)\notin E_{j_0}\text{ $j_0$ witness to $i$},\\
\Rightarrow& Ante(d)\subseteq E_{j_0},\:\neg Conc(d)\notin E_{j_0}\nonumber\\
\Rightarrow& Ante(d)\subseteq E,\:\neg Conc(d)\notin E_{j_0}\:.\nonumber
\end{align}

\noindent Clearly, this means $Ante(d)\subseteq E$.

Now assume for contradiction that $\neg Conc(d)\in E$, which means there is some $i_0\in\nat$ such that $\neg Conc(d)\in E_{i_0}$.

What is the relationship between $i_0$ and $j_0$? As both are natural numbers, there are three possibilities:
\begin{itemize}
\item $i_0=j_0$: This is impossible as else we will have $\neg Conc(d)\notin E_{i_0}$ and $\neg Conc(d)\in E_{i_0}$.
\item $i_0<j_0$: We have $\neg Conc(d)\notin E_{j_0}$ and $\neg Conc(d)\in E_{i_0}$, which is also impossible because the $E_i$'s form an ascending chain, so $E_{i_0}\subseteq E_{j_0}$. Therefore, we get $\neg Conc(d)\in E_{j_0}$ and $\neg Conc(d)\notin E_{j_0}$.
\item $i_0>j_0$: We have $\neg Conc(d)\notin E_{j_0}$ and $\neg Conc(d)\in E_{i_0}$. From Equation \ref{eq:intermediate_d_is_active}, we have that $d$ is active in $E_{j_0}$, hence $Conc(d)\in E_{j_0+1}\subseteq E_{i_0}$, which makes $\neg Conc(d)\in E_{i_0}$ impossible because the $E_i$'s are consistent.
\end{itemize}

\noindent Therefore, $\neg Conc(d)\notin E$.

So we have $Ante(d)\subseteq E$ and $\neg Conc(d)\notin E$. Therefore, Equation \ref{eq:NBD_proof_inter} is true for this case and we have
\begin{align}\label{eq:GD_subset_alternative_form}
GD\subseteq\set{d\in D\:\vline\:Ante(d)\subseteq E,\:\neg Conc(d)\notin E}\:.
\end{align}

\noindent We can take the union of Equations \ref{eq:GD_subset_alternative_form} and \ref{eq:SAD_subset_alternative_form} to get
\begin{align}\label{eq:half_the_alternative_form1}
GD\cup SAD\subseteq\set{d\in D\:\vline\:Ante(d)\subseteq E,\:\neg Conc(d)\notin E}\:.
\end{align}

($\Leftarrow$) Assume $d\in\set{d'\in D\:\vline\:Ante(d')\subseteq E,\:\neg Conc(d')\notin E}$, which means $Ante(d)\subseteq E$ and $\neg Conc(d)\notin E$. We have, for some $i_0\in\nat$,

\begin{align*}
\Leftrightarrow& Ante(d)\subseteq E_{i_0},\:\pair{\forall j\in\nat}\neg Conc(d)\notin E_j\\
\Leftrightarrow& Ante(d)\subseteq E_{i_0},\:\neg Conc(d)\notin E_{i_0},\:\pair{\forall j\in\nat-\set{i_0}}\neg Conc(d)\notin E_j\\
\Leftrightarrow&\pair{\forall j\in\nat-\set{i_0}}\neg Conc(d)\notin E_j\text{ and }\\
&[\pair{Ante(d)\subseteq E_{i_0},\:\neg Conc(d)\notin E_{i_0},\:Conc(d)\in E_{i_0}}\text{ or }\\
&\pair{Ante(d)\subseteq E_{i_0},\:\neg Conc(d)\notin E_{i_0},\:Conc(d)\notin E_{i_0}}]\\
\Leftrightarrow&\pair{\forall j\in\nat-\set{i_0}}\neg Conc(d)\notin E_j\text{ and }\\
&[\pair{Ante(d)\subseteq E_{i_0},\:\neg Conc(d)\notin E_{i_0},\:Conc(d)\in E_{i_0}}\text{ or }d\in GD_{i_0}\\
\Rightarrow&\pair{\forall j\in\nat-\set{i_0}}\neg Conc(d)\notin E_j\text{ and }\\
&[\pair{Ante(d)\subseteq E_{i_0},\:\neg Conc(d)\notin E_{i_0},\:Conc(d)\in E_{i_0}}\text{ or }d\in GD\\
\Rightarrow& d\in GD\text{ or }[Ante(d)\subseteq E_{i_0},\:\neg Conc(d)\notin E_{i_0},\:Conc(d)\in E_{i_0}\text{ and }\\
&\pair{\forall j\in\nat-\set{i_0}}\neg Conc(d)\notin E_j]\\
\Rightarrow& d\in GD\text{ or }\sqbra{Ante(d)\subseteq E\text{ and }\pair{\forall j\in\nat}\neg Conc(d)\notin E_j}\\
\Rightarrow& d\in GD\text{ or }\sqbra{Ante(d)\subseteq E\text{ and }\neg Conc(d)\notin E}\:,\\
\Leftrightarrow& d\in GD\cup SAD\:.
\end{align*}
\noindent Therefore, we have
\begin{align}\label{eq:half_the_alternative_form2}
\set{d\in D\:\vline\:Ante(d)\subseteq E,\:\neg Conc(d)\notin E}\subseteq GD\cup SAD\:.
\end{align}

\noindent The result follows from Equations \ref{eq:half_the_alternative_form1} and \ref{eq:half_the_alternative_form2}.
\end{proof}

\end{document}